\documentclass[10pt,journal,compsoc]{IEEEtran}
\usepackage{amsfonts}
\usepackage{bm}
\usepackage{amsthm}
\usepackage{wrapfig}
\usepackage{enumitem}
\usepackage{multirow}
\usepackage{array}
\usepackage{ragged2e}
\usepackage{booktabs}
\newtheorem{assumption}{Assumption}

\newtheorem{theorem}{Theorem}
\newtheorem{lemma}{Lemma}

\newtheorem{fact}{Fact}
\theoremstyle{remark}
\newtheorem*{remark}{Remark}

\usepackage[]{subfig}
\usepackage{hyperref}
\usepackage{url}
\usepackage{algorithm}
\usepackage[noend]{algpseudocode}

 \usepackage{color}

\ifCLASSOPTIONcompsoc
  \usepackage[nocompress]{cite}
\else
  \usepackage{cite}
\fi
\ifCLASSINFOpdf
   \usepackage[pdftex]{graphicx}
   \graphicspath{{../pdf/}{../jpeg/}}
   \DeclareGraphicsExtensions{.pdf,.jpeg,.png}
\else
\fi
\usepackage{amsmath}
\usepackage{array}
\begin{document}
%
\title{Achieving Personalized Federated Learning with Sparse Local Models}
%
%
%
%

\author{Tiansheng~Huang,~
        Shiwei~Liu,~
        Li~Shen,~
        Fengxiang~He,~
        Weiwei~Lin,~
        and~Dacheng~Tao,~\IEEEmembership{Fellow,~IEEE}
\IEEEcompsocitemizethanks{
\IEEEcompsocthanksitem This work was done when Tiansheng Huang and Shiwei Liu worked as interns at JD Explore Academy. Li Shen is the corresponding author.
\IEEEcompsocthanksitem T. Huang and W. Lin are with the Department
of Computer Science and Engineering, South China University of Technology, Guangzhou, China.\protect\\
E-mail: tianshenghuangscut@gmail.com, linww@scut.edu.cn.
\IEEEcompsocthanksitem S. Liu is with Eindhoven University of Technology, Netherlands.\protect\\
E-mail: s.liu3@tue.nl.
\IEEEcompsocthanksitem L. Shen, F. He and D. Tao are with JD Explore Academy, Beijing, China.\protect\\
E-mail: mathshenli@gmail.com,  fengxiang.f.he@gmail.com, dacheng.tao@gmail.com.}
\thanks{Manuscript received xx xx, xxxx; revised xx xx, xxxx.}}

%
%

\markboth{Journal of \LaTeX\ Class Files,~Vol.~xx, No.xx, xx~xxxx}%
{Shell \MakeLowercase{\textit{et al.}}: Bare Demo of IEEEtran.cls for Computer Society Journals}
%



\IEEEtitleabstractindextext{%
\begin{abstract}
Federated learning (FL) is vulnerable to heterogeneously distributed data, since a common global model in FL may not adapt to the heterogeneous data distribution of each user. To counter this issue, personalized FL (PFL) was proposed to produce dedicated local models for each individual user. However, PFL is far from its maturity, because existing PFL solutions either demonstrate unsatisfactory generalization towards different model architectures or cost enormous extra computation and memory.  In this work, we propose federated learning with personalized sparse mask (FedSpa), a novel PFL scheme that employs personalized sparse masks to customize sparse local models on the edge. Instead of training an intact (or dense) PFL model, FedSpa only maintains a fixed number of active parameters throughout training (aka sparse-to-sparse training), which enables users' models to achieve personalization with cheap communication, computation, and memory cost. 
We theoretically show that the iterates obtained by FedSpa converge to the local minimizer of the formulated SPFL problem at the rate of  $\mathcal{O}(\frac{1}{\sqrt{T}})$. 
Comprehensive experiments demonstrate that FedSpa significantly saves communication and computation costs, while simultaneously achieves higher model accuracy and faster convergence speed against several state-of-the-art PFL methods.
\end{abstract}

\begin{IEEEkeywords}
Dynamic sparse training, federated learning, model compression, personalized federated learning. 
\end{IEEEkeywords}}

\maketitle

\IEEEdisplaynontitleabstractindextext

%
\IEEEpeerreviewmaketitle

\IEEEraisesectionheading{\section{Introduction}\label{sec:intro}}

Data privacy raises increasingly intensive concerns, and governments have enacted legislation to regulate the privacy intrusion behavior of mobile users, e.g., the General Data Protection Regulation \cite{voigt2017eu}. 
Traditional distributed learning approaches, requiring massive users' data to be collected and transmitted to a central server for training, soon may no longer be realistic under the increasingly stringent regulations on users' private data.
\par 
On this ground, federated learning (FL), a distributed training paradigm emerges as a successful solution to cope with privacy concerns, which allows multiple clients to perform model training within the local device without the necessity to exchange the data to other entities. 
In this way, the data privacy leakage problem could be potentially relieved.
\begin{figure}[!t]
	\centering
	\includegraphics[width=3.2in]{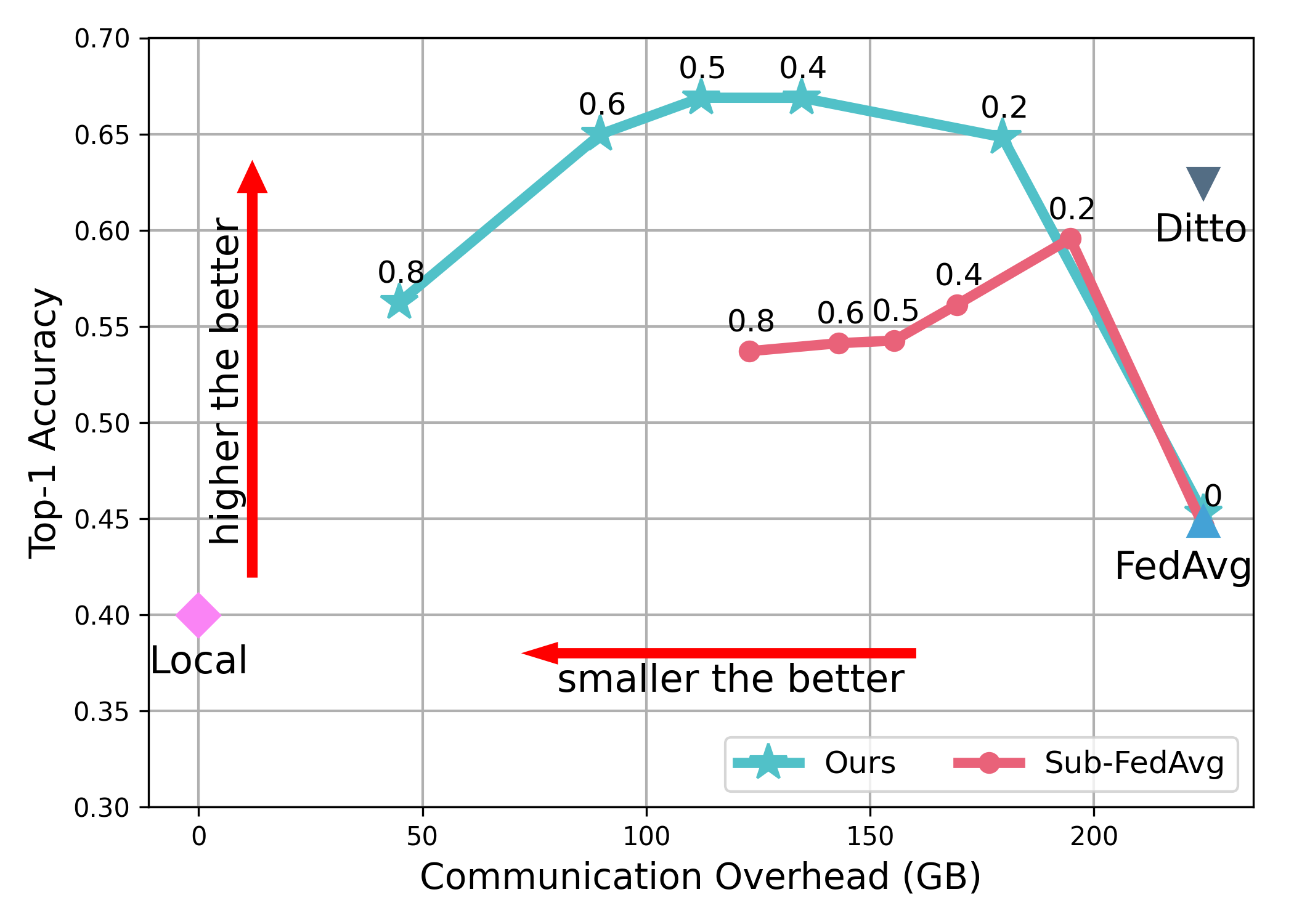}
	\caption{Performance of FedSpa and several baselines w.r.t. communication cost in Non-IID setting. Numbers above FedSpa and Sub-FedAvg are sparsity. We show that the optimal sparsity (that attains the highest accuracy) is between 0.4 and 0.5. In other words, accuracy is increased while  communication and computation overhead during training are reduced.}
	\label{performace-FedSpa}
	\vspace{-20pt}
\end{figure}
\par
\begin{figure*}
	\centering
	\vspace{-40pt}
	\includegraphics[ width=19cm]{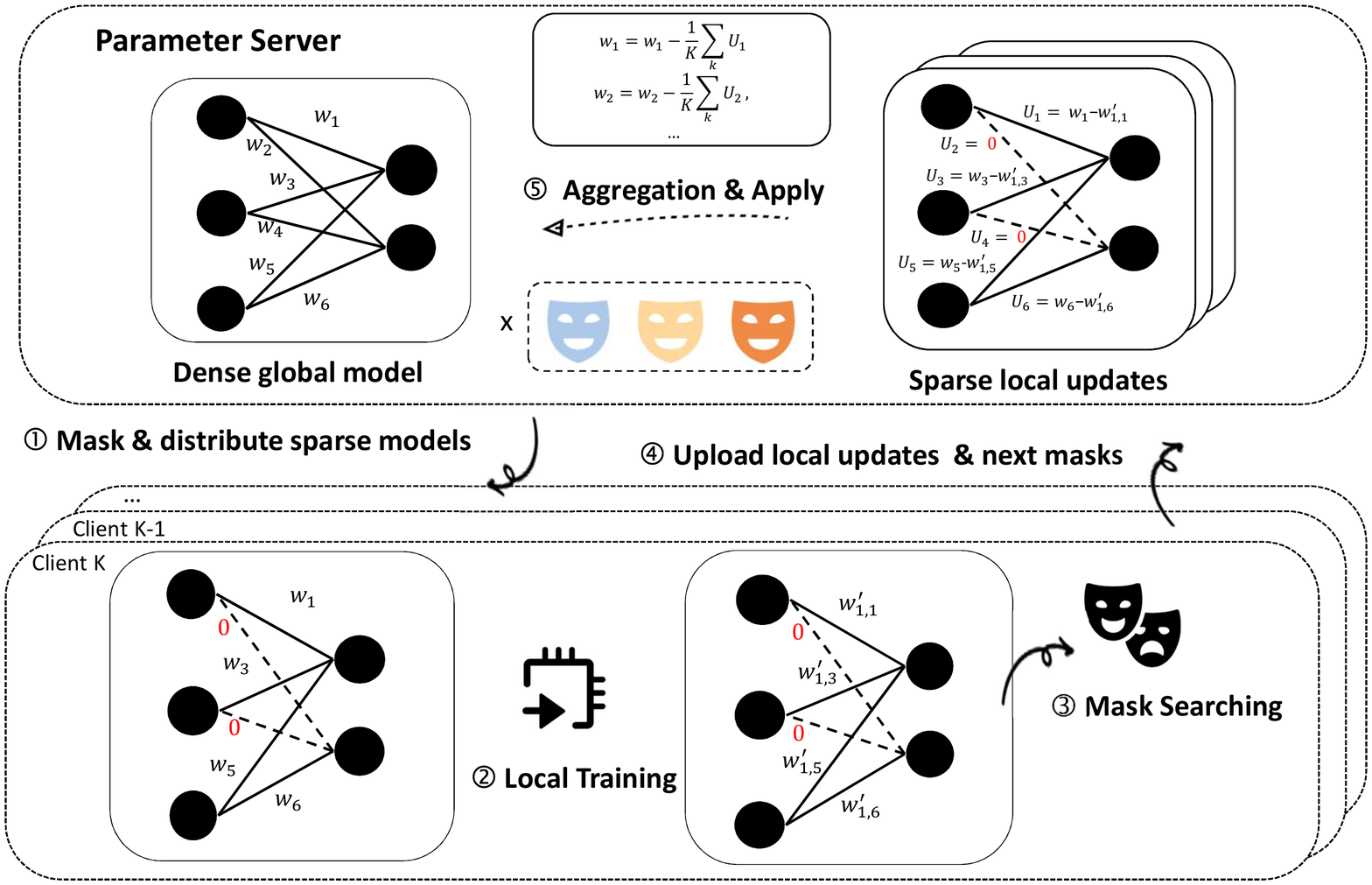}
	\vspace{-50pt}
	\caption{Overview of FedSpa. Firstly, the server multiplies the personalized masks and the global model to produce sparse models, which are then distributed to clients. Secondly, clients do local training on a constantly sparse model. Thirdly, clients search for new personalized masks based on the model after trained and their local data.  Finally, clients upload the sparse gradient updates and the new masks to server, and the gradient updates are aggregated and applied to the global model. }
	\label{overview of FedSpa}
\end{figure*}

Despite the promising prospect, several notorious issues are afflicting  practical performance of FL:
\par
	 {\color{black}\emph{The global model produced by weight average (or FedAvg and its non-personalized variants) exhibits unsatisfactory performance in a Non-IID data distribution setting.}
		To alleviate this problem, the most popular idea is to integrate personalized features into the global model, and produce dedicated model for each local distribution.} However, how to make this integration is an open problem that remains unresolved. Prior works on personalized FL (PFL) zero in this issue, but the existing methods either demonstrate weak generalization towards different model architectures \cite{arivazhagan2019federated}, or require extra computation and storage \cite{li2021ditto}.

	  \emph{The communication and training overhead is prohibitively high for both the FL and PFL}. Clients in FL/PFL responsible for model training are mostly edge-devices with limited computation capacity and low bandwidth, and may not be powerful enough to fulfill a modern machine learning task with large deep neural networks.
	Existing studies \cite{li2020lotteryfl,vahidian2021personalized} integrate model compression into FL/PFL to save communication and computation overhead.  However, both methods embrace the technique of dense-to-sparse training, which still requires a large amount of communication at the beginning of training. In addition, how to effectively aggregate the dynamic sparse models is another challenging problem that remains unresolved. 
In this work, we propose FedSpa (see Figure \ref{overview of FedSpa}), which has two key features to counter the above two challenges: ({\bf i}) FedSpa does not deploy a single global model, but allows each client to own its unique sparse model masked by a personalized mask, which successfully alleviates the Non-IID  challenge. ({\bf ii})  FedSpa allows each client to train over an evolutionary sparse model with constant sparsity\footnote{\color{black} Sparsity specifies the ratio of parameters that are set to 0 (or inactive) in a model.} throughout the whole federated training process, which consistently alleviates the computation overhead of clients. 
Besides, all the local models in FedSpa are sparse models, which requires a smaller amount of communication cost in each communication round.   
Theoretically, we conclude that the proposed solution can achieve sub-linear convergence towards the formulated problem.
Empirically, in the Non-IID setting, we demonstrate that FedSpa \textbf{accelerates} the convergence (respectively 76.2\% and 38.1\% less communication rounds to reach the best accuracy of FedAvg \cite{mcmahan2016communication} and Ditto \cite{li2021ditto}), \textbf{increases} the final accuracy (up to 21.9\% and 4.4\% higher accuracy than FedAvg and Ditto, respectively), \textbf{reduces} the communication overhead (50\% less parameters communicated than the dense solutions),   and \textbf{lowers} the computation (15.3\% lower floating-point operations (FLOPs) than algorithms trained with fully dense model). To the end, we summarize our contribution as:
\begin{itemize}
	\item  We present a novel formulation of the sparse personalized FL (SPFL) problem, which {\color{black} can be applied} to various network architectures by enforcing personalized sparse masks to a global model. 
	\item  We propose a solution dubbed as FedSpa to solve the SPFL problem. By our novel design, FedSpa reduces the communication and computation overhead of the general FL solution. 
	\item Two sparse-to-sparse mask searching techniques are integrated as plugins of our solution. To adapt our PFL training context, we modify the DST-based mask searching technique to enable a warm-start of the searching process, which achieves superior performance. 
	\item We theoretically show that FedSpa obtains the convergence rate in the scale of $\mathcal{O}(\frac{1}{\sqrt{T}})$.  Experimental results conducted on different datasets and network models also demonstrate the superiority of FedSpa.
\end{itemize}
The remainder of this paper is organized as follows. In Section 2, we first provide a brief taxonomy over the recent works on PFL, and we subsequently discuss  previous research on sparse training, a key component of our proposed solution.  In Section 3, We formalize the sparse PFL (SPFL) problem  on which we make a brief discussion. In Section 4, our solution dubbed as FedSpa is proposed. Two mask-searching techniques are proposed and integrated into the FedSpa framework. In the same section, we make a conclusion on FedSpa's theoretical convergence property. Experiment results are given  and briefly discussed in Section 5. At last, conclusion and future prospect are made in Section 6.  
\section{Related Works} 

Federated learning (FL) \cite{mcmahan2016communication} is seriously afflicted by the issue of heterogeneously distributed (or Non-IID) data. Personalized FL (PFL), initiated by recent literature \cite{li2021ditto,arivazhagan2019federated}, is shown to be effective to counter this issue of FL. In this work, we propose an alternative yet effective way to enhance PFL with personalized sparse models.

\subsection{Personalized Federated Learning}
We categorize PFL into five genres.

Firstly, PFL via layer partition, e.g., FedPer \cite{arivazhagan2019federated}, LG-FedAvg \cite{liang2020think}, FedRep \cite{collins2021exploiting},  is to divide the global model layers into shared layers and personalized layers. For the shared layers, weights average as in FedAvg is adopted, while for personalized layers, models are trained only locally and will not be exchanged with others. 

Secondly, PFL via regularization, e.g., Ditto \cite{li2021ditto}, L2GD \cite{hanzely2020federated} is to add a proximal term on the local model to force the local model and global model closely in the local model fine-tuning stage. 


Thirdly, PFL via model interpolation, e.g., MAPPER \cite{mansour2020three}, APFL \cite{deng2020adaptive1}  achieves personalization by linearly interpolating the weights of the cluster (global) model and local model as the personalized model. 

Fourthly, PFL via transfer learning,  e.g.,  FedMD \cite{li2019fedmd}, FedSteg \cite{yang2020fedsteg},  and Fedhealth \cite{chen2020fedhealth}, is to either use model and domain-specific local fine-tuning or knowledge distillation to adapt the global model into the personalized model.

Finally, PFL via model compression, e.g.,  LotteryFL \cite{li2020lotteryfl} and Sub-FedAvg \cite{vahidian2021personalized},  achieves personalization via employing principle model compression techniques, such as weight pruning and channel pruning, over the shared global model.

\subsection{Sparse Deep Neural Networks}
Methods to sparsify neural networks can be classified into two genres: dense-to-sparse methods and sparse-to-sparse methods. 

Dense-to-sparse methods train from a dense model, and compress the model along the training process. Iterative pruning, first proposed by \cite{frankle2018lottery}, shows promising performance in dynamically searching for a sparse yet accurate network. 

Recently, sparse-to-sparse methods have been proposed to pursue training efficiency. Among them, dynamic sparse training (DST)~\cite{bellec2018deep,evci2020rigging,liu2021we} is the most successful technique that allows sparse networks, trained from scratch, to match the performance of their dense equivalents. Stemming from the first work -- sparse evolutionary training~\cite{mocanu2018scalable,liu2020sparse}, DST has evolved as a class of sparse training methods absorbing many advanced techniques, e.g., weight redistribution~\cite{mostafa2019parameter,dettmers2019sparse}, gradient-based regrowth~\cite{dettmers2019sparse,evci2020rigging}, and extra weight exploration~\cite{jayakumar2020top,liu2021we}.

\subsection{Discussion on model-compression-based PFL}
Our work also achieves personalization via model compression. We emphasize that three main progresses are made towards SOTA compression-based PFL: ({\bf i}) We rigorously formulate the sparse personalized FL problem, filling the gap left by the prior works. ({\bf ii}) While prior works either vaguely describe their model aggregation as  "aggregating the Lottery Ticket Network via FedAvg" \cite{li2020lotteryfl}, or "taking the average on the intersection of unpruned parameters in the network" \cite{vahidian2021personalized}, we explicitly formulate the aggregation as averaging the sparse update from clients. ({\bf iii}) Both the two prominent prior works utilize the idea of iterative pruning to prune the network from dense to sparse. We instead provide two sparse-to-sparse training alternatives to plug in our solution, which largely reduces the costs of communication at the beginning of the training process, and exhibits remarkable performance. 


\section{Problem Formulation}

We assume a total number of $K$ clients within our FL system, and we consistently use $k$ to index a specific client. First, we give a preliminary introduction on the general FL problem. 

\textbf{General FL problem.} 
Let $\bm w \in \mathbb{R}^d $ be the global weight. General FL takes the formulation as below
\begin{equation*}
	\label{fedavg ultimate}
	\begin{split}
		& \text{(P1)} \quad \min_{\bm w} \tilde{f}(\bm w)= \frac{1}{K}\sum_{k=1}^{K}  \tilde{F}_{k}(\bm w)   \\ 
		\text{s.t.} \quad  &\tilde{F}_{k}(\bm w) = \mathbb{E} [ \mathcal{L}_{(\bm x,  y) \sim \mathcal{D}_k}(\bm w; (\bm x,  y))]
	\end{split}
\end{equation*}
where $\mathcal{D} = \mathcal{D}_1\cup \dots \cup\mathcal{D}_K$ is the joint distribution of $k$ local heterogeneous distributions,  $(\bm x, y)$ denotes one piece of data that is  uniformly sampled wrt distribution $\mathcal{D}_k$ . $\mathcal{L}(\cdot;\cdot)$ is the loss corresponds to the model weights and data.

\textbf{Sparse PFL problem.} 
By introducing personalized masks into FL,  we alternatively derive the SPFL problem as follows:

\begin{equation*}
	\label{fl goal}
	\begin{split}
		\text{(P2)} \qquad &\min_{\bm w}  {f}(\bm w)=\frac{1}{K} \sum_{k=1}^{K}   {F}_{k}( \bm m_k^* \odot \bm w )   ,  \\  
		\text{s.t.} \quad  & {F}_{k}( \bm m_k^* \odot \bm w ) = \mathbb{E}[  \mathcal{L}_{(\bm x, y) \sim \mathcal{D}_k}( \bm m_k^* \odot \bm w; (\bm x, y)) ]
	\end{split}
\end{equation*}
where $\bm m_k^* \in \{0,1\}^d$ is a personalized sparse binary mask for $k$-th client.
$\odot$ denotes the Hadamard product for two given vectors.
Our ultimate goal is to find a global model $\bm w$, such that the personalized model for $k$-th client can be extracted from the global model by personalized mask $\bm m_k^*$, i.e., $\bm m_k^* \odot \bm w$. The element of $\bm m_k^*$ being $1$ means that the weight in the global model is active for $k$-th personalized model, otherwise, remains dormant.  Thus, the information exchange between all personalized models is enforced by a shared global model $\bm w$.

Compared with existing PFL algorithms, solving our SPFL problem (P2) does not sacrifice additional computation and storage overhead of clients, since we do not maintain both personalized local models and global model in clients as  \cite{li2021ditto,mansour2020three}. On contrary, the solution to our problem could potentially lower the communication and computation overhead. Moreover, our prposed SPFL problem (P2) {\color{black} can be applied to most of the model architectures without model-specific hyper-parameter tuning}, since we do not make model-specific separation of the public and personalized layer as in \cite{arivazhagan2019federated,liang2020think,collins2021exploiting}, or domain-specific fine-tuning as in \cite{chen2020fedhealth,yang2020fedsteg}.

\section{FedSpa: solution for SPFL}  
In this section, we first introduce our proposed FedSpa in Algorithm \ref{algorithm1}. Then,  we specify the update rule of global model, and two sparse-to-sparse mask searching methods that can be plugged in the update process. At last, we give a theoretical analysis on evaluating the quality of the iterates of FedSpa with respect to the ultimate PFL problem (P2). 

\subsection{ Global Model Update  for FedSpa}
\label{sgd solution}
\textbf{\color{black} Data Parallel-based Update.} 
We first propose the following iterative update to solve problem (P2) 
\begin{equation}
	\label{one step average trivial}
	\begin{split}
		\bm w_{t+1} =  \bm w_{t} -  \frac{\eta}{K} \sum_{k=1}^K  \bm m_{k}^* \odot \nabla_{\tilde{\bm w}_{k,t} } \mathcal{L}(  \tilde{\bm w}_{k,t}; \xi_{k,t}),
	\end{split}
\end{equation}
where $\xi_{k,t}$ is a batch of data that is uniformly sampled from the $k$-th client's local distribution $\mathcal{D}_k$, $\eta_t$ is the learning rate for iteration $t$, and  $\tilde{\bm w}_{k,t} = \bm m_{k}^* \odot \bm w_{t}$ is the sparse weights sparsified by mask $\bm m_{k}^*$. However, the optimal personalized masks $\{\bm m_k^*\}$ are generally not accessible to us in the solution process.
Let $\bm m_{k,t}$ be an intermediate surrogate personalized mask of $\bm m_{k}^*$. We  subsequently rewrite Eq.~(\ref{one step average trivial}) as follows:
\begin{equation}
	\label{one step average}
	\begin{split}
		\bm w_{t+1} =  \bm w_{t} -  \frac{\eta}{K} \sum_{k=1}^K  \bm m_{k,t} \odot \nabla_{\tilde{\bm w}_{k,t} } \mathcal{L}(  \tilde{\bm w}_{k,t}; \xi_{k,t}).
	\end{split}
\end{equation}
For our proposed update rule, it is worth mentioned that: ({\bf i}) Some coordinates of the model weights have been made zero before doing the forward process, i.e., not all the parameters have to be involved when calculating  $\mathcal{L}( \tilde{\bm w}_{k,t};\xi_{k,t})$. This means that the computation overhead in the forward process could be potentially saved. ({\bf ii}) In the backward process, the stochastic gradient $\nabla_{\tilde{\bm w}_{k,t} } \mathcal{L}( \tilde{\bm w}_{k,t};\xi_{k,t})$ is masked again by $\bm m_{k,t}$, which means that we do not need to backward the gradient for those sparse coordinates. Thus, the computation cost can be largely saved.

{\color{black}\textbf{FL-adapted Update.}}
To save the communication overhead, we integrate the idea from  \textit{local SGD} \cite{stich2018local} and \textit{partial participation} to our solution.  Let $\tilde{\bm w}_{k,t,\tau}$ denote the weights before doing $(\tau+1)$-th step of local SGD and set $\tilde{\bm w}_{k,t,0}= \bm m_{k,t} \odot \bm  w_t$ (i.e., the local weights will be synchronized every $N$ steps with the global weights). Then for each local step  $\tau=0, 1, \dots, N-1,$,  each client $k \in S_t$ updates its model as below:
\begin{equation}
	\begin{split}
		\label{local SGD update2}
		\tilde{\bm w}_{k, t, \tau+1}  = \tilde{\bm w}_{k,t, \tau}- \eta \bm m_{k,t} \odot \nabla_{\tilde{\bm w}_{k,t,\tau} } \mathcal{L}( \tilde{\bm w}_{k,t,\tau}; \xi_{k,t,\tau}),  
	\end{split}
\end{equation}
where $\xi_{k,t,\tau}$ is a batch of sampled data in $\tau$-th step at round $t$. After the local training is finished, the models of participated clients are updated and aggregated to the global model in server as follows: 
\begin{equation}
	\begin{split}
		\label{aggregation2}
		\bm w_{t+1} &= \bm w_{t} -  \frac{1}{ | S_t| } \sum_{k \in S_t}   (\tilde{\bm w}_{k,t,0} - \tilde{\bm w}_{k,t,N}),
	\end{split}
\end{equation}
where $S_t$ is the set of clients selected to be participant in round $t$. 
According to Eq.~(\ref{local SGD update2}), the update synchronized to the server (i.e., $\tilde{\bm w}_{k,t,0} - \tilde{\bm w}_{k,t,N}$), and the model distributed to clients (i.e., $\tilde{\bm w}_{k,t,0}$) are all sparse with a constant sparsity. Therefore, the communication overhead over synchronization could be largely saved. At last, we summarize our proposed FedSpa in Algorithm \ref{algorithm1}.

\begin{algorithm}[!hbtp]
	\caption{FedSpa}
	\label{algorithm1}
	{\color{black}
		\hspace*{\algorithmicindent} \textbf{Input} Training iteration $T$; Learning rate $\eta$;  Local Steps $N$; Random seed $seed$;}
	\begin{algorithmic}[1]
		\Procedure{Server's Main Loop}{}
		\State Randomly initialize global model $\bm w_0$
		\State  $\bm m_{k,0}$ =  MaskInit({\color{black}$seed$})  for $k\!=\!0,1,\dots,K$  
		\For { $t = 0, 1, \dots, T-1$}
		\State Uniformly sample a fraction of client into $S_t$ 	
		\For {each client $k \notin S_t$}
		\State $\bm m_{k,t+1} = \bm m_{k,t}$  \Comment{Inherit masks for round $t+1$ if not chosen }
		\EndFor
		\For { $k \in S_{t}$}
		\State Send $ \tilde{\bm w}_{k,t,0} = \bm m_{k,t} \odot \bm w_t$ to client $k$
		\State Call Client $k$'s main loop and receive  $\bm U_{k,t}$ and $ \bm m_{k,t+1}$
		\EndFor
		\State 	$\bm w_{t+1} = \bm w_{t} -  \frac{1}{S} \sum_{k \in S_t}  \bm U_{k,t}$   \Comment{Average and apply the update}
		\EndFor
		\EndProcedure
		\Procedure{Client's Main Loop}{}
		\For {$\tau = 0, 1, \dots,N-1$}
		\State Sample a batch of data $\xi_{k,t,\tau}$ from local dataset 
		\State $\bm g_ {k,t,\tau}(\tilde{\bm w}_{k,t,\tau}) =  \nabla_{\tilde{\bm w}_{k,t,\tau} } \mathcal{L}( \tilde{\bm w}_{k,t,\tau}; \xi_{k,t,\tau}) $
		\State $ \tilde{\bm w}_{k, t, \tau+1}  = \tilde{\bm w}_{k,t, \tau}- \eta_t \bm m_{k,t} \odot \bm g_ {k,t,\tau}(\tilde{\bm w}_{k,t,\tau}) $ \Comment{Local update with fixed mask $\bm m_{k,t}$}
		\EndFor
		\State $\bm U_{k,t} =  \tilde{\bm w}_{k,t,0} - \tilde{\bm w}_{k,t,N}  $
		\State $\bm m_{k,t+1}$ = Next\_Masks($\cdot$)   \Comment{Plug in a mask-searching solution to produce next masks}
		\State Send back $\bm U_{k,t}$ and $\bm m_{k,t+1}$ 
		\EndProcedure
	\end{algorithmic}
\end{algorithm}

\subsection{Sparse-to-sparse Mask Searching Technique}
The framework of FedSpa is extensible. {\color{black} Specifically, we use the mask surrogate $\bm m_{k,t}$ in Eq. (\ref{local SGD update2}) to perform the update, which allows us to plug in arbitrary mask searching techniques to determine the iterating process of $\bm m_{k,t}$.} In this work, we nominate two kinds of sparse-to-sparse training techniques:
modified dynamic sparse training (DST)  and Random Static Masks (RSM) into FedSpa to search for the optimal local masks.


\textbf{Modified  DST for FedSpa.}
Our modified DST solution (see Algorithm \ref{algorithm3}) for FL follows these procedures. Firstly, randomly initialize the same mask for each client based on \text { Erdós-Rényi Kernel} (ERK) \cite{evci2020rigging}.  Secondly, after local training, each client prunes out a number of unpruned weights with the smallest magnitude, and the number of weights being pruned is determined by a decayed pruned rate. Thirdly, recover the same amount of weights pruned in the last step. We follow the recovery process as in \cite{evci2020rigging} by utilizing the gradient information to do the recovery. By our DST method, the number of sparse weights (aka. \textit{sparse volume}) remains a constant (i.e., $\beta$) throughout the whole training process.

\begin{algorithm}[!hbtp]
	\caption{Modified DST for FedSpa}
	\label{algorithm3}
	{\color{black} \hspace*{\algorithmicindent} \textbf{Input} Initial pruning rate $\alpha_0$; Set of Model layers $\mathcal{J}$;}
	\begin{algorithmic}[1]
		
		\Procedure{MaskInit}{{\color{black}$seed$}}
		\State   Randomly initialize  $\bm m_{k,0}$ using the same random seed $seed$.
		\EndProcedure
		\Procedure{Next\_Masks}{$\tilde{\bm w}_{k,t,N}$}
		\State  Decay $\alpha_{t}$ using cosine annealing with initial pruning rate $\alpha_0$
		\State  Sample a batch of data and backward the dense gradient $g(\tilde{\bm w}_{k,t,N})$
		\For { layer $j \in \mathcal{J}$}
		\State Update mask $\bm m_{k,t+\frac{1}{2}}^{(j)}$ by zeroing out $\alpha_t$-proportion of weights with magnitude pruning
		\State Update mask  $\bm m_{k,t+1}^{(j)}$ via recovering weights with gradient information $g(\tilde{\bm w}_{k,t,N})$
		\EndFor
		\State Return $\bm m_{k,t+1}$
		\EndProcedure
	\end{algorithmic}
\end{algorithm}
\begin{remark}
	We highlight our main modification over traditional DST techniques like Rigl \cite{evci2020rigging} and Set (\cite{mocanu2018scalable}) to an FL context exist in two main aspects: ({\bf i}) The pruning is performed individually by each client based on their local models, and the gradient used for weights recovery is derived using the client's local training data.  ({\bf ii}) Once the next masks are generated,  existing DST solutions immediately apply them to the local model weight. Indicated by \cite{liu2021we}, by doing so, the recovered coordinate may need extra training steps to grow from 0 to a dense value.  Our solution relieves this problem by applying the new mask on the global weights (which are dense), such that the recovered coordinates could have a dense initial value to warm-start.
\end{remark}


		\begin{algorithm}[H]
			\caption{RSM for FedSpa}
			\label{algorithm2}
			\begin{algorithmic}[1]
				\Procedure{MaskInit}{{\color{black}$seed$}}
				\State  Randomly initialize  $\bm m_{k,0}$ using the same random seed {\color{black}$seed$}
				\EndProcedure
				\Procedure{Next\_Masks}{}
				\State $\bm m_{k,t+1} = \bm m_{k,t} $ 
				\State Return $\bm m_{k,t+1}$
				\EndProcedure
			\end{algorithmic}
		\end{algorithm}

\textbf{RSM for FedSpa.} 
RSM (shown in Algorithm \ref{algorithm2}) is basically fixing $\bm m_{k,t}$ for all $k \in [K]$ to the same randomly initialized mask, which remains unchanged during the whole training session.  This solution also ensures the same sparse volume for all the clients throughout the training process, and could also reduce the computation and communication overhead as DST. Interestingly, within the setting of  the homogeneous data distribution, we empirically show that RSM is more effective than DST in FedSpa. 

\subsection{Theoretical Analysis}
In this section, we shall introduce the convergence property of  FedSpa. We first give the following assumptions to enable further analysis.
\begin{assumption} [Bounded gradient dissimilarity between sparse models]
	\label{bounded gradient dissimilarity}
	For any $\tilde{\bm w} \in \mathbb{R}^d$, there exists a constant $G \geq 0$ bounding the  gradient dissimilarity over all clients, i.e., $\left\|  \bm m_k^* \nabla F_{k}(\tilde{\bm {w}}) -  \frac{1}{K}\sum_{k^{\prime}} \bm m_{k^{\prime}}^* \nabla F_{k^{\prime}}(\tilde{\bm {w}}) \right\| \leq G$.
\end{assumption}
\begin{assumption}[Unbiased estimator and bounded variance over sparse masks]
	\label{bounded variance}
	For $\tilde{\bm w} \in \mathbb{R}^d$, assume that $\bm g_{k,t,\tau}(\tilde{\bm w})\!:=\!\nabla \mathcal{L}(\tilde{\bm w}; \xi_{k,t,\tau} )$ is an unbiased estimator of $\nabla F_k(\tilde{\bm w})$. Additionally, for $\bm m_{k,t} \in \{0,1\}^d, k \in [K] , t \in [T], \tau \in [N], \tilde{\bm{w}} \in \mathbb{R}^d$, the variance over sparse masks satisfies: $\mathbb{E}\left[\left\| \bm m_{k,t} \odot \bm g_{k,t,\tau}(\tilde{\bm w})\!-\! \bm m_{k,t} \odot\nabla F_k(\tilde{\bm w})\right\|^{2})\right] \!\leq\! \sigma^{2}$.
\end{assumption}
\begin{assumption}[L-smoothness]
	\label{L-smoothness}
	We assume L-smoothness over the client's loss function, i.e., $\|\nabla F_k (\tilde{\bm {w}}_1)-\nabla F_k (\tilde{\bm {w}}_2)\| \leq L\|\tilde{\bm w}_1- \tilde{\bm w}_2\|$ holds for arbitrary $\tilde {\bm w}_1,\tilde{\bm w_2} \in \mathbb{R}^d$.
\end{assumption}
\begin{assumption}[Bounded gradient]
	\label{bounded gradient}
	Suppose the  gradient of global loss over arbitrary models $\bm w \in \mathbb{R}^d$ is upper-bounded, i.e.,
	$\|\nabla f( \bm w)  \| \leq B.
	$
\end{assumption}
Assumptions \ref{L-smoothness} and \ref{bounded gradient} are commonly used for characterizing the convergence of FL algorithms.  
We modify assumption \ref{bounded gradient dissimilarity} and \ref{bounded variance} slightly from their counterparts in existing FL literature \cite[see their Assumptions 4 and 5]{xu2021fedcm} in order to reveal the variance over sparse masks and the  gradient heterogeneity between the local sparse models. 
For Assumption 2, similar formulation can be found in  \cite[Assumption 2]{mohtashami2021simultaneous}.
\begin{theorem}[{\color{black}Convergence property of FedSpa}] \label{convergence theorem} Given the above assumptions, suppose the learning rate satisfies $\eta_t \leq \sqrt{\frac{1}{16 L^2 N^2 p_t T }}$ and $p_t = \max_{k \in [K], \tau \in [N]} \left \{ \frac{ || \bm m_{k,t} \odot \nabla F_k( \tilde{\bm w}_{k,t,\tau}) -    \bm m_{k}^* \odot \nabla F_k (  \tilde{\bm w}_{k,t})  ||^2}{\|\nabla F_k( \tilde{\bm w}_{k,t,\tau}) -\nabla F_k( \tilde{\bm w}_{k,t}) \|^2} \right\}$, FedSpa exhibits the following {\color{black} convergence property} towards the SPFL problem (P2): 

\begin{equation}
\label{convergence-bound}
	\begin{split}
		\frac{1}{T}\sum_{t=0}^{T-1}	\mathbb{E} ||    \nabla f \left(   {\bm {w}}_{t} \right) ||^2 & \leq \frac{\hat{p} V +\tilde{p} C_1+C_2 }{ \sqrt{T} }      \\
	\end{split}
\end{equation}
where $\hat{p} = \max_{t \in [T]} \sqrt{p_t} $, $\tilde{p}=\frac{1}{T}\sum_{t=0}^{T-1} \frac{1}{\sqrt{p_t}} $, $V= 8L ( f \left(  {{\bm w}}_{0} \right)  -  f \left(  {{\bm w}}^*   \right))$, $ C_1 = \frac{25}{32NT} (\sigma^2+6NG^2)+\frac{75B^2}{16T} + \frac{3G^2}{2}   + \frac{   \sigma^2}{2NS}+\frac{5B^2}{2}   $, and $C_2 = \frac{5}{16N \sqrt{T} }(\sigma^2+6NG^2)+ \frac{15B^2}{8 \sqrt{T}}$.
\end{theorem}
\begin{remark}
  The above result corroborates that the iterates obtained by FedSpa converges to the local minimizer of problem (P2) at the rate of $\mathcal{O}(1/\sqrt{T})$. Another critical observation is that the quality of convergence is closely related to $p_t$. Or in other words, the surrogate masks obtained via the mask searching method could have critical impact on the convergence of FedSpa. Specifically, the constant term in the bound could escalate to infinite if $p_t \to \infty$ or $p_t \to 0$. Both cases indicate that the surrogate masks $m_{k,t}$ are seriously drifted from the local optimal masks $m_k^*$, implying that the mask searching techniques are yielding unsatisfactory performance.
\end{remark}

\section{Experiments}\label{experiments}
In this section, we conduct extensive experiments to verify the efficacy of the proposed FedSpa. 
Our implementation of FedSpa is based on an open-source FL simulator FedML \cite{he2020fedml}.  
\begin{figure*}[!t]
	\centering
	\includegraphics[ width=18cm]{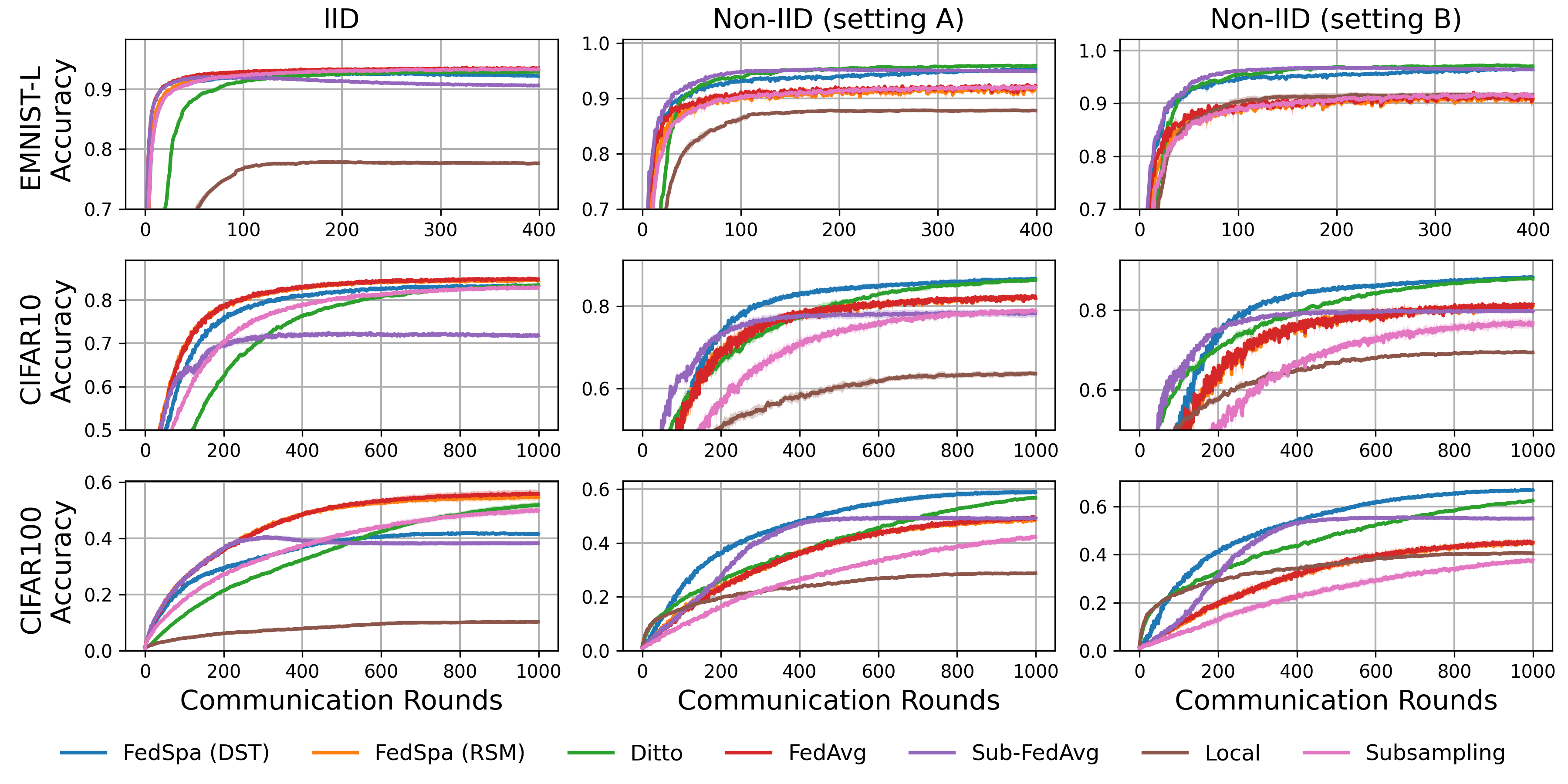}
	\caption{Test Accuracy vs. Communication Rounds  }
	\label{test accuracy main}
\end{figure*}
We fix the dense ratio of FedSpa (DST), FedSpa (RSM), and the final dense ratio of Fed-SubAvg both to 0.5 (i.e., 50\% of parameters are pruned) in our main evaluation. Other hyper-parameters are fixed as default. Figure \ref{test accuracy main} and Table \ref{main table} illustrate the training performance of different algorithms on three datasets. We evaluate the performance based on the following metrics:

\begin{table*}[!hbtp]
	\caption{Table illustrating performance of different methods. }
	\label{main table}
	\resizebox{\textwidth}{!}{%
		\begin{tabular}{ccccc|ccc|ccc}
			\toprule
			\multirow{4}{*}{Task} & \multirow{4}{*}{Method} & \multicolumn{3}{c}{\multirow{2}{*}{IID}} & \multicolumn{6}{c}{Non-IID} \\ \cmidrule{6-11} 
			&  & \multicolumn{3}{c}{} & \multicolumn{3}{c}{Setting A} & \multicolumn{3}{c}{Setting B} \\   \cmidrule(lr){3-5} \cmidrule(lr){6-8} \cmidrule(lr){9-11} 
			&  & Acc & Comm Cost  & FLOPs & Acc & Comm Cost & FLOPs & Acc & Comm Cost & FLOPs
			\\ &  &  &  (GB) & (1e14) &  &  (GB) & (1e16) &  &  (GB) & (1e14) 
			\\ \midrule
			\multicolumn{1}{c}{\multirow{6}{*}{\begin{tabular}[c]{@{}c@{}}EMNIST-L\\ (LeNet)\end{tabular}}} 
			&FedSpa (DST) & 92.2$\pm$0.1 & \textbf{7.0} & 2.0 & 95.3$\pm$0.1 & \textbf{7.0} & 2.0 & 96.5$\pm$0.2 & \textbf{7.0} & 2.0\\ 
			&FedSpa (RSM) & 92.9$\pm$0.1 & \textbf{7.0} & 2.0 & 91.9$\pm$0.2 & \textbf{7.0} & 2.0 & 90.6$\pm$0.9 & \textbf{7.0} & 2.0\\ 
			&Ditto & 92.9$\pm$0.1 & 14.1 & 3.5 & \textbf{95.9}$\pm$0.1 & 14.1 & 3.5 & \textbf{97.0}$\pm$0.2 & 14.1 & 3.5\\ 
			&FedAvg & \textbf{93.5}$\pm$0.2 & 14.1 & 3.5 & 92.3$\pm$0.3 & 14.1 & 3.5 & 90.9$\pm$0.8 & 14.1 & 3.5\\ 
			&Sub-FedAvg & 90.7$\pm$0.2 & 9.5 & \textbf{1.9} & 94.9$\pm$0.2 & 9.4 & \textbf{1.9} & 96.4$\pm$0.2 & 9.4 & \textbf{1.9}\\ 
			&Local & 77.6$\pm$0.3 & - & 3.5 & 87.8$\pm$0.1 & - & 3.5 & 91.6$\pm$0.5 & - & 3.5 \\
			&\color{black}Subsampling & {\color{black}93.3$\pm$0.2} & {\color{black}10.5} & {\color{black}3.5} & {\color{black}92.0$\pm$0.4} & {\color{black}10.5} & {\color{black}3.5} & {\color{black}91.3$\pm$0.6 }& {\color{black}10.5} & {\color{black}3.5}\\ 
			\midrule
			\multirow{6}{*}{\begin{tabular}[c]{@{}l@{}}CIFAR-10\\ (VGG11) \end{tabular}} 
			&FedSpa (DST) & 83.4$\pm$0.1 & \textbf{369.2} & 172.9 & \textbf{86.6}$\pm$0.5 & \textbf{369.2} & 173.3 & \textbf{88.2}$\pm$0.4 & \textbf{369.2} & 173.5\\ 
			&FedSpa (RSM) & 84.5$\pm$0.1 & \textbf{369.2} & 172.9 & 82.1$\pm$0.2 & \textbf{369.2} & 173.3 & 80.9$\pm$0.2 & \textbf{369.2} & 173.5\\ 
			&Ditto & 83.5$\pm$0.2 & 738.5 & 229.3 & 86.4$\pm$0.6 & 738.5 & 229.8 & 87.8$\pm$0.3 & 738.5 & 230.0\\ 
			&FedAvg & \textbf{84.8}$\pm$0.3 & 738.5 & 229.3 & 82.0$\pm$0.4 & 738.5 & 229.8 & 81.4$\pm$0.4 & 738.5 & 230.0\\ 
			&Sub-FedAvg & 71.8$\pm$0.3 & 410.2 & \textbf{121.4} & 78.3$\pm$1.0 & 424.7 & \textbf{120.6} & 79.6$\pm$0.6 & 416.9 & \textbf{119.8}\\ 
			&Local & 42.5$\pm$0.2 & - & 229.3 & 63.6$\pm$0.6 & - & 229.8 & 69.4$\pm$0.2 & - & 230.0\\
			&{\color{black}Subsampling} & {\color{black} 83.0$\pm$0.4} & {\color{black}553.9} & {\color{black}229.3} & {\color{black}78.9$\pm$0.5} & {\color{black}553.9} & {\color{black}229.8} & {\color{black}76.7$\pm$1.0 } & {\color{black}553.9} & {\color{black}230.0}\\ 
			\midrule
			\multirow{6}{*}{\begin{tabular}[c]{@{}l@{}}CIFAR-100\\ (ResNet18)\end{tabular}} 
			&FedSpa (DST) & 41.5$\pm$0.5 & \textbf{448.8} & 705.1 & \textbf{59.0}$\pm$1.0 & \textbf{448.8} & 704.9 & \textbf{66.9}$\pm$0.2 & \textbf{448.8} & 704.8\\ 
			&FedSpa (RSM) & 54.6$\pm$1.1 & \textbf{448.8} & 705.1 & 48.7$\pm$0.5 & 448.8 & 704.9 & 44.6$\pm$0.5 & \textbf{448.8} & 704.8\\ 
			&Ditto & 51.9$\pm$1.1 & 897.6 & 833.2 & 56.8$\pm$0.6 & 897.6 & 833.0 & 62.5$\pm$0.2 & 897.6 & 832.9\\ 
			&FedAvg & \textbf{55.7}$\pm$1.3 & 897.6 & 833.2 & 49.3$\pm$0.4 & 897.6 & 833.0 & 45.0$\pm$0.9 & 897.6 & 832.9\\ 
			&Sub-FedAvg & 38.3$\pm$0.8 & 616.5 & \textbf{494.1} & 49.2$\pm$0.7 & 624.4 & \textbf{508.4} & 55.0$\pm$0.7 & 612.8 & \textbf{496.1}\\
			&Local & 10.3$\pm$0.3 & - & 833.2 & 28.8$\pm$0.1 & - & 833.0 & 40.5$\pm$0.4 & - & 832.9\\
			&  {\color{black} Subsampling} &  {\color{black}49.8$\pm$1.3 }& {\color{black}673.2 }& {\color{black}833.2} & {\color{black}42.3$\pm$0.8} &{\color{black} 673.2} & {\color{black}833.0} & {\color{black}37.6$\pm$1.1} & {\color{black}673.2} & {\color{black}832.9} \\ 
			\hline
		\end{tabular}%
	}
\end{table*}
\begin{table*}[htbp]
	\caption{Communication rounds to a fixed accuracy.   }
	\label{tab:communication to accuracy}
	\resizebox{\textwidth}{!}{%
		\begin{tabular}{llll|lll|lll}
			\midrule
			\multirow{3}{*}{CIFAR10} & \multicolumn{3}{c}{\multirow{2}{*}{IID}} & \multicolumn{6}{c}{Non-IID} \\ \cmidrule{5-10} 
			& \multicolumn{3}{c}{} & \multicolumn{3}{c}{Setting A} & \multicolumn{3}{c}{Setting B} \\ \cmidrule(lr){2-4} \cmidrule(lr){5-7} \cmidrule(lr){8-10} 
			& Acc@70 & Acc@75 & Acc@80 & Acc@70 & Acc@75 & Acc@80 & Acc@70 & Acc@75 & Acc@80 \\ \midrule
			FedSpa (DST) & 134.0$\pm$2.9 & 183.3$\pm$6.8 & 312.3$\pm$16.2 & 167.3$\pm$4.0 & \textbf{210.3}$\pm$4.2 & \textbf{281.3}$\pm$19.1 & 164.3$\pm$5.0 & 206.3$\pm$4.1 & \textbf{270.0}$\pm$5.1\\ 
			FedSpa (RSM) & \textbf{101.3}$\pm$1.7 & 141.3$\pm$6.2 & 237.0$\pm$6.4 & 195.3$\pm$10.7 & 271.3$\pm$16.2 & 471.7$\pm$19.8 & 252.0$\pm$12.7 & 339.0$\pm$20.6 & 614.0$\pm$72.8\\ 
			Ditto & 284.7$\pm$8.1 & 370.3$\pm$9.3 & 549.3$\pm$22.6 & 242.3$\pm$12.7 & 334.0$\pm$16.5 & 466.3$\pm$30.3 & 190.3$\pm$6.1 & 278.0$\pm$22.0 & 417.7$\pm$10.2\\ 
			FedAvg & 105.0$\pm$2.2 & \textbf{140.3}$\pm$4.7 & \textbf{228.7}$\pm$23.5 & 198.3$\pm$14.8 & 256.7$\pm$10.8 & 474.7$\pm$31.4 & 241.0$\pm$3.7 & 327.3$\pm$8.5 & 583.7$\pm$65.5\\ 
			Sub-FedAvg & 197.7$\pm$22.9 & $>1000$ & $>1000$ & \textbf{151.7}$\pm$10.6 & 235.0$\pm$17.1 & $>1000$ & \textbf{137.3}$\pm$1.7 & \textbf{191.7}$\pm$6.6 & $>1000$\\
			{\color{black} Subsampling} & {\color{black}198.0$\pm$4.5 }& {\color{black}268.3$\pm$4.5 }& {\color{black}457.0$\pm$13.5 }& {\color{black}365.3$\pm$23.8 }& {\color{black}523.0$\pm$43.4 }& {\color{black}$>1000$ }& {\color{black}466.3$\pm$15.0 }& {\color{black}722.7$\pm$105.1 }& {\color{black}$>1000$}\\ 
			\midrule
			CIFAR100 
			& Acc@40 & Acc@50 & Acc@55 & Acc@40 & Acc@50 & Acc@55 & Acc@40 & Acc@50 & Acc@55 \\ 
			\midrule
			FedSpa (DST) & 536.3$\pm$35.9 & $>1000$ & $>1000$ & \textbf{236.3}$\pm$12.3 & \textbf{442.0}$\pm$16.9 & \textbf{595.0}$\pm$41.3 & \textbf{181.3}$\pm$7.8 & \textbf{314.7}$\pm$17.4 & \textbf{407.7}$\pm$16.7\\ 
			FedSpa (RSM) & \textbf{239.3}$\pm$4.1 & \textbf{435.7}$\pm$25.8 & $>1000$ & 460.7$\pm$12.5 & $>1000$ & $>1000$ & 594.0$\pm$10.7 & $>1000$ & $>1000$\\ 
			Ditto & 545.7$\pm$19.4 & 868.7$\pm$56.1 & $>1000$ & 455.3$\pm$6.9 & 724.0$\pm$20.2 & 894.0$\pm$25.0 & 301.3$\pm$10.8 & 534.0$\pm$11.3 & 678.7$\pm$7.9\\ 
			FedAvg & 245.0$\pm$5.1 & 436.3$\pm$25.3 & $>1000$ & 470.7$\pm$25.4 & $>1000$ & $>1000$ & 589.7$\pm$46.7 & $>1000$ & $>1000$\\ 
			Sub-FedAVG & $>1000$ & $>1000$ & $>1000$ & 280.7$\pm$2.5 & $>1000$ & $>1000$ & 246.3$\pm$10.1 & 335.3$\pm$14.1 & 511.0$\pm$85.9\\ 
			{\color{black}Subsampling} & {\color{black}460.7$\pm$16.6} & {\color{black}$>1000$} & {\color{black}$>1000$} & {\color{black}845.3$\pm$59.2} & {\color{black}$>1000$} & {\color{black}$>1000$ }& {\color{black}$>1000$ }& {\color{black}$>1000$ }& {\color{black}$>1000$}\\ 
			\hline
		\end{tabular}%
	}
\end{table*}
\subsection{Experimental Setup}
\label{setup}
\textbf{Dataset.} We evaluate the efficacy of FedSpa on EMNIST-Letter (EMNIST-L henceforth), CIFAR10, and CIFAR100 datasets. We simulate the client's data distribution on Non-IID and IID setting. We simulate two groups of Non-IID settings via $\gamma$-Dirichlet distribution, named setting A and setting B. Setting A and setting B respectively specify $\gamma=0.2,0.1$ for both EMNIST-L and CIFAR100, while specify  $\gamma=0.5,0.3$ for CIFAR10.  Details of our simulation setting are available in Appendix \ref{Data Splitting Setting}.

\textbf{Baselines.} We compare our proposed FedSpa with four baselines, including  FedAvg \cite{mcmahan2016communication}, Sub-FedAvg \cite{vahidian2021personalized}, Ditto \cite{li2021ditto} and Local. We tune the hyper-parameters of the baselines to their best states.  Specifically, the regularization factor of Ditto is set to 0.5. The prune rate each round, distance threshold, and accuracy threshold of Fed-Subavg are fixed to 0.05, 0.0001, 0.5, respectively. We ran 3 random seeds in our comparison.

\textbf{Models and hyper-parameters.} We use LeNet5 for EMNIST-L, VGG11 for CIFAR10, and ResNet18 for CIFAR100 in our experiment. We use a SGD optimizer with weight decayed parameter 0.0005. The learning rate is initialized with 0.1 and decayed with 0.998 after each communication round. 
We simulate 100 clients in total, and in each round 10 of them are picked to perform local training (the setting follows \cite{mcmahan2016communication}).   
For all the methods except Ditto, local epochs are fixed to 5. For Ditto, in order to ensure a fair comparison, each client uses 3 epochs for training of the local model, and 2 epochs for global model training. The batch size of all the experiments is fixed to 128. 
For FedSpa, the pruning rate (i.e., $\alpha_{t}$) is decayed using cosine annealing with an initial pruned rate 0.5. The initial sparsity of layers is initialized by ERK with scale parameter 1. 

\subsection{Main Performance Evaluation}

\textbf{Final Accuracy.} 
In the Non-IID setting, we show that FedSpa (DST) achieves remarkable performance. Specifically, in Non-IID setting B of CIFAR100, FedSpa (DST)  achieves respectively 4.4\%, 11.9\% and 21.9\% higher final model accuracy, compared with Ditto, Sub-FedAvg and FedAvg. FedSpa (DST) seems to achieve better performance as the FL tasks becoming difficult (since better performance is observed in a higher Non-IID extent, and in datasets that are intrinsically more difficult). Interestingly, in the IID setting, we show that all the personalized solutions exhibit some extents of performance degradation, which become more significant as the dataset becomes challenging. The compression-based methods seem to be especially vulnerable in this setting. Our interpretation for this phenomenon is that: since the information exchange between clients would be limited by employing different sub-networks for training, the clients could not efficiently make an effective fusion on their models through parameter averaging. This hypothesis is substantiated by our experiment on  FedSpa (RSM), an alternative implementation of FedSpa, which forces all the masks to maintain the same sub-network.  FedSpa (RSM) achieves commensurate performance with FedAvg in the IID setting, outperforming the personalized solutions.

\textbf{Convergence.} As shown in Table \ref{tab:communication to accuracy}, FedSpa achieves significantly faster convergence, which potentially saves the communication rounds to train a model from scratch to a specific accuracy.

\textbf{Training FLOPs and Communication.}  
From Table \ref{main table}, FedSpa (DST) achieves 15.4\%$\sim$42.9\% lower  FLOPs than the dense solutions (e.g., Ditto, FedAvg), 13.0\%$\sim$28.2\% lower communication overhead than another model compression solution Sub-FedAvg, and  50\% lower communication than the dense solution. The edge of FedSpa (DST) stems from its training pattern -- it is trained from a sparse model, with constant sparsity throughout the training process. However, it is interesting to see that the training FLOPS of Sub-FedAvg is considerably lower than FedSpa, even under the same sparsity setting. This phenomenon stems from our ERK initialization, which is essential for the high performance of our solution, for which we will have a further discussion in our ablation study.


\subsection{Ablation Study}
\label{ablation} 
In this sub-section, we give and discuss the experimental results of the ablation study of FedSpa. Specifically, we study the impacts of dense ratio, different mask initialization methods, and the gradient-involved weight recovery procedure. Additionally, we present an interesting observation on the performance of the global model trained by our personalized solution. Our ablation study is done with ResNet-18 on CIFAR100. 

\begin{figure*}[!hbtp]
	\centering
	\includegraphics[width=7in]{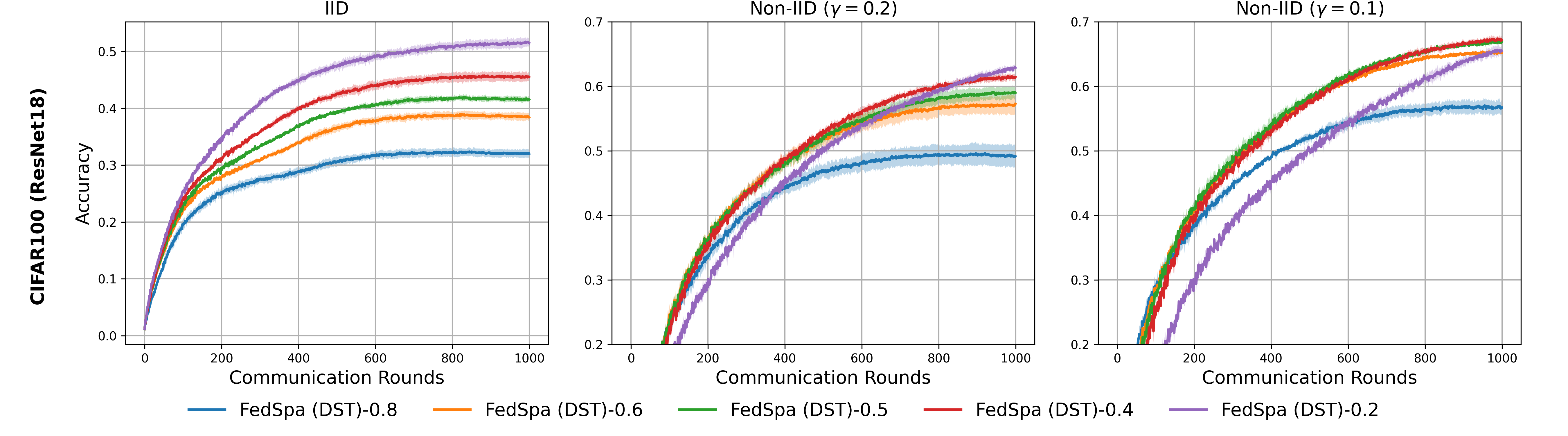}
	\caption{ FedSpa (DST) under different sparsity. Numbers in the labels are sparsity. }
	\label{impact of dense ratio}
\end{figure*}
\begin{table*}[htbp]
	\centering
	\caption{Performance of FedSpa (DST) under different sparsity settings.}
	\resizebox{\textwidth}{!}{%
		\begin{tabular}{cccccccccc} 
			\hline
			\multirow{3}{*}{Sparsity} & \multicolumn{3}{c}{\multirow{2}{*}{iid}} & \multicolumn{6}{c}{Non-iid} \\ 
			\cline{5-10}
			& \multicolumn{3}{c}{} & \multicolumn{3}{c}{$\gamma$=0.2} & \multicolumn{3}{c}{$\gamma$=0.1} \\ 
			\cline{2-10}
			& Acc & Comm Cost & FLOPs & Acc & Comm Cost & FLOPs & Acc & Comm Cost& FLOPs \\ 
			&  &  (GB) & (1e16) &  & (GB) & (1e16) &  & (GB) & (1e16) \\ 
			\hline
			0.2 & \textbf{51.5}$\pm$0.8 & 718.1 & 8.2 & \textbf{62.9}$\pm$0.4 & 718.1 & 8.2 & 65.5$\pm$0.5 & 718.1 & 8.2\\ 
			0.4 & 45.5$\pm$0.9 & 538.6 & 7.6 & 61.4$\pm$0.6 & 538.6 & 7.5 & \textbf{67.2}$\pm$0.4 & 538.6 & 7.5\\ 
			0.5 & 41.5$\pm$0.5 & 448.8 & 7.1 & 59.0$\pm$1.0 & 448.8 & 7.0 & 66.9$\pm$0.2 & 448.8 & 7.0\\ 
			0.6 & 38.4$\pm$0.6 & 359.0 & 6.5 & 57.3$\pm$1.5 & 359.0 & 6.5 & 65.2$\pm$0.2 & 359.0 & 6.5\\ 
			0.8 & 32.0$\pm$0.7 & \textbf{179.5} & \textbf{4.6} & 49.2$\pm$1.8 & \textbf{179.5} & \textbf{4.6} & 56.7$\pm$0.8 & \textbf{179.5} & \textbf{4.6}\\ 
			\hline
			\label{table: sparsity setting}
	\end{tabular}}
\end{table*}
\textbf{Impact of sparsity (aka. sparse ratio).}  
Fixing other components and hyper-parameters to the default value in our setup, we change the sparsity of FedSpa to 0.2, 0.4, 0.5, 0.6 and 0.8, to show its impact on the algorithm performance. Experimental results are available in Figure \ref{impact of dense ratio} and Table \ref{table: sparsity setting}. By our report, we observe that sparsity may impact learning performance under different data distribution settings. For the IID setting, a higher sparsity seems to seriously degrade the training performance, while for the Non-IID setting, properly sparsifying the model may even enhance the final accuracy, and with a higher Non-IID extent, the benefit of sparsification reinforces.  But too much sparsity, even in the highly non-iid setting (e.g. $\gamma=0.1$) leads to performance degradation.  On the contrary, while it is iid, the convergence could be dominated by the errors brought by sparsification, and setting the mask to a higher sparsity could possibly enlarge these existing errors. 

An intuitive interpretation for the impact of sparse ratio is from the perspective of information exchange. Too much sparsification may limit the information exchange between the local sparse models.  If all the clients maintain an extremely high sparsity, the intersected coordinates between clients' local sparse models (or identically, their masks) would be small. Then the local update averaging process (see Eq. (\ref{aggregation2})), the only way to extract global knowledge into the local models, would not be effective. On contrary, while the sparsity is set to an extremely low value, the personalized features of local models could be eliminated, since only limited coordinates in their models are different.   \par

\par

\textbf{ERK vs. Uniform sparsity initialization.} Recall our mask initialization procedure in Algorithm \ref{algorithm2} that the layer-wise sparsity is initialized by ERK. This in essence ensures that the layer-wise sparsity of a model is scaled with the number of parameters in a layer. \cite{liu2021we} confirms the outstanding effect of ERK initialization in improving overall training performance over the centralized training primitive, but it remains unexplored how it performs in our proposed distributed training framework. Below, we show in Figure \ref{Initialization ablation1} how accuracy evolves with communication rounds under ERK and Uniform \footnote{Uniform enforces the same sparsity for all the layers in a model.} initialization. As shown, a drastic drop of accuracy is observed by replacing ERK with Uniform, by which we conclude that ERK is an essential component for FedSpa (DST).

\begin{figure*}[!hbtp]
	\centering
	\includegraphics[width=7in]{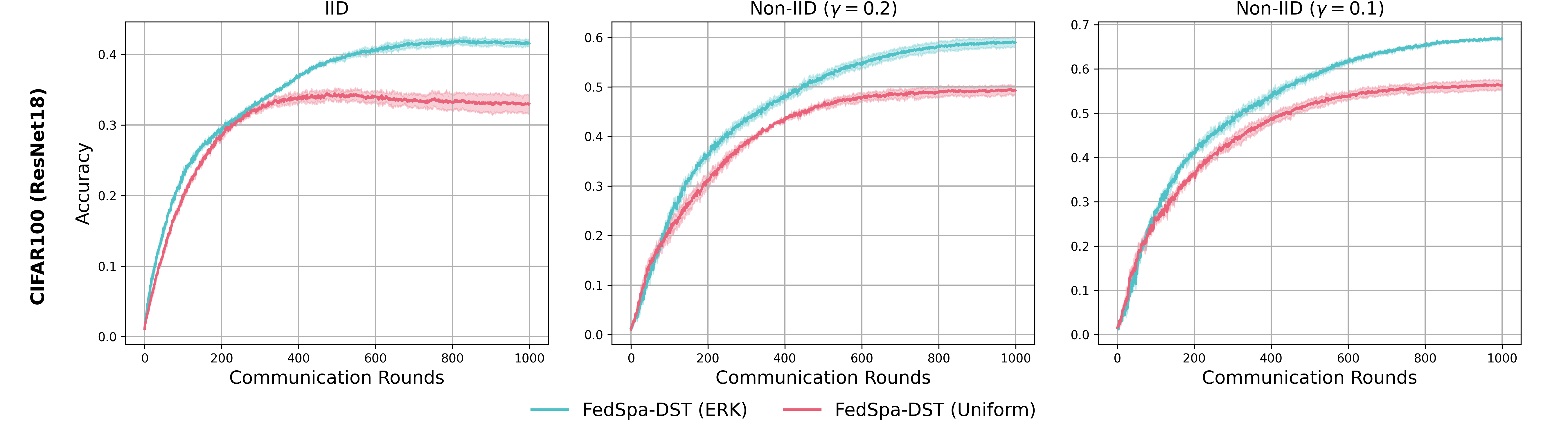}
	\caption{Layer-wise sparsity initialized by ERK or Uniform. Sparsity of FedSpa is fixed to 0.5.  }
	\label{Initialization ablation1}
\end{figure*}
However, though a significant accuracy enhancement is observed, we note that integrating ERK may sacrifice potentially more FLOPS reduction. This observation can be found in Table \ref{Initialization ablation2}, wherein our results show that initialization with Uniform can save 34.3\% FLOPs of that with ERK. 

\begin{table*}[!hbtp]
	\centering
	\caption{ Performance of FedSpa  under ERK and Uniform initialization. }
	\resizebox{\textwidth}{!}{%
		\begin{tabular}{cccccccccc} 
			\hline
			\multirow{3}{*}{Methods} & \multicolumn{3}{c}{\multirow{2}{*}{iid}} & \multicolumn{6}{c}{Non-iid} \\ 
			\cline{5-10}
			& \multicolumn{3}{c}{} & \multicolumn{3}{c}{$\gamma$=0.2} & \multicolumn{3}{c}{$\gamma$=0.1} \\ 
			\cline{2-10}
			& Acc & Comm Cost & FLOPs & Acc & Comm Cost & FLOPs & Acc & Comm Cost& FLOPs \\ 
			&  &  (GB) & (1e16) &  & (GB) & (1e16) &  & (GB) & (1e16) \\ 
			\hline
			ERK & \textbf{41.5}$\pm$0.5 & 448.8 & 7.1 & \textbf{59.0}$\pm$1.0 & 448.8 & 7.0 & \textbf{66.9}$\pm$0.2 & 448.8 & 7.0\\ 
			Uniform & 33.0$\pm$1.4 & 448.8 & \textbf{4.6} & 49.3$\pm$0.9 & 448.8 & \textbf{4.6} & 56.3$\pm$1.1 & 448.8 & \textbf{4.6}\\ 
			\hline
	\end{tabular}}
\end{table*}
\par
\textbf{Different or same mask initialization.} 
Recall that based on the layer-wise sparsity calculated by ERK, FedSpa uses the same random seed to initialize the mask, so as to make the mask exploration of all clients started from the same mask. In the following, we give another implementation that allows each client to share different masks in the beginning. \par
As shown in Figure \ref{Initialization ablation2}, we surprisingly find that for FedSpa (DST), maintaining different masks in initialization may slightly enhance its training performance in IID and Non-IID ($\gamma=0.2$) setting. We hypothesize that by different mask initialization, each client could more efficiently search for their optimal masks that better represents the features and labels of the personal data.  
\par
\begin{figure*}[!hbtp]
	\centering
	\includegraphics[width=7in]{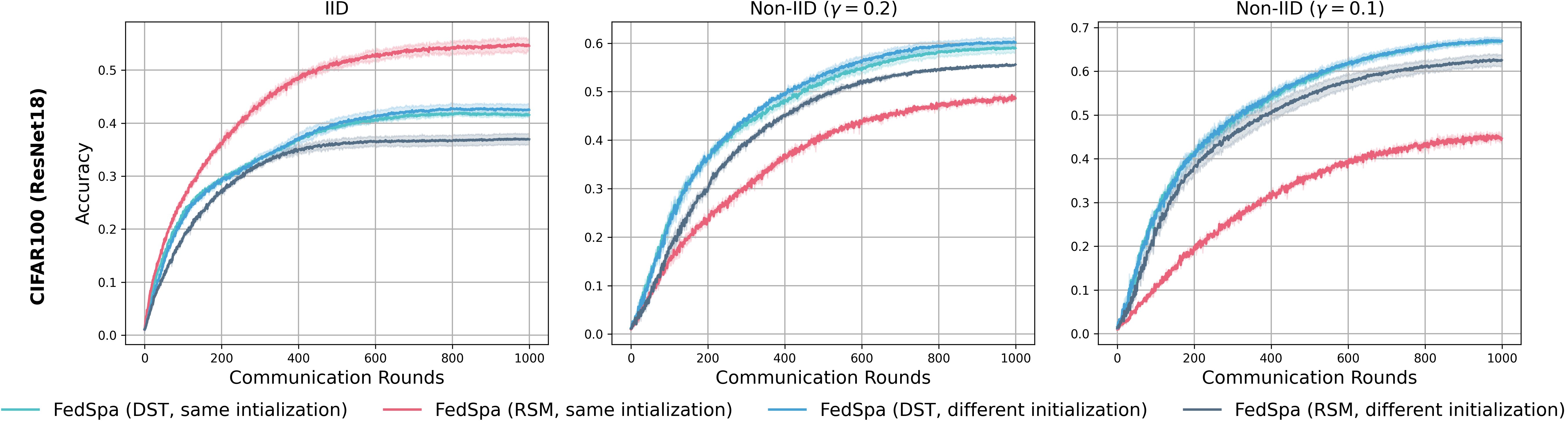}
	\caption{Initialization based on same or different masks. Sparsity of FedSpa is fixed to 0.5.  }
	\label{Initialization ablation2}
\end{figure*}
For FedSpa (RSM), compared with initialization using the same mask, different mask initialization may result in a drastic performance loss in the IID setting, and a significant improvement in the Non-IID setting. With the same mask initialization of RSM,  each client consistently trains based on the same sub-network, which completely eliminates personalization. So this setting shares a similar performance with FedAvg -- with satisfactory performance in IID setting and rather weak performance in Non-IID setting. On contrary, by initializing different masks in the beginning, FedSpa (RSM) reserves some degrees of personalization, since only the intersected coordinates in their local models are shared and updated by the information exchange (i.e., average) process. Consequently. FedSpa (RSM) with different mask initializations has a similar performance pattern with FedSpa (DST).  \par

Another interesting observation is that FedSpa (RSM) with different mask initialization cannot outperform FedSpa (DST) in both the two groups of Non-IID settings. This indicates that the DST mask searching process is effective to achieve a superior performance of FedSpa in Non-IID setting. 

\begin{figure*}[!hbtp]
	\centering
	\includegraphics[width=7in]{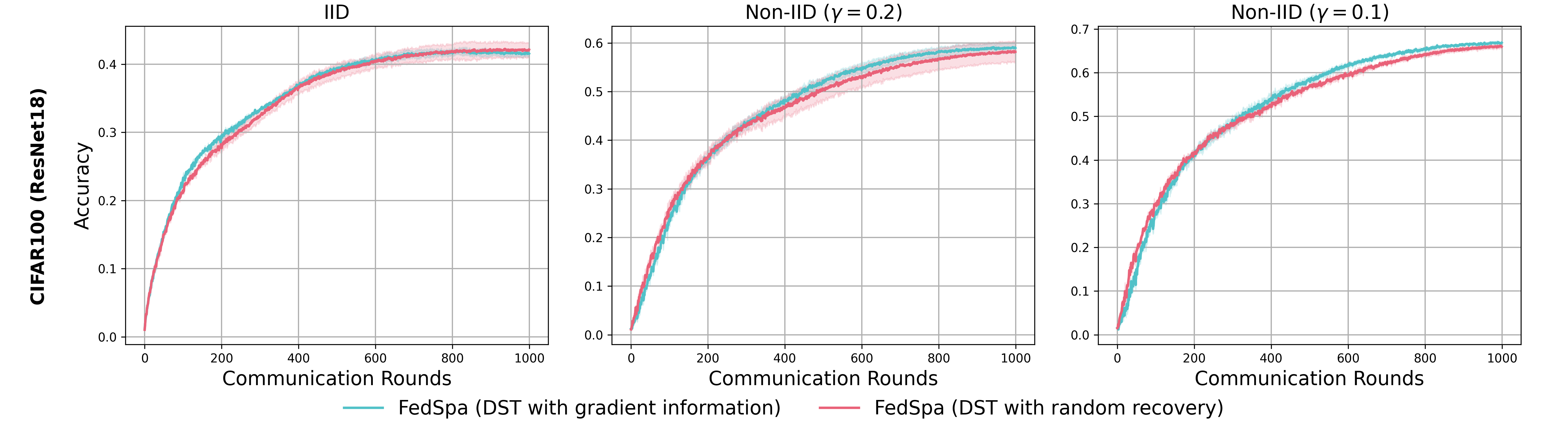}
	\caption{Recovery with gradient information or random recovery. Sparsity is fixed to 0.5.  }
	\label{Rigl ablation}
\end{figure*}
\textbf{Weight recovery w/ or w/o gradient information.} 
Recall that in FedSpa (DST), we proposed to use gradient information to recover the pruned weights, which is empirically proven in \cite{evci2020rigging} to outperform its random recovery counterpart in Set \cite{mocanu2018scalable}. Specifically, for gradient information-based recovery, the weight coordinates with the top-$\alpha_t$ magnitude of the gradient would be recovered, while for random recovery, the coordinates are recovered randomly. To demonstrate the impact of the weight recovery method over FedSpa (DST), in Figure \ref{Rigl ablation}, we compare the gradient information-based recovery with random recovery.  Our experimental result demonstrates that recovery with gradient information could slightly accelerate the convergence and enhance the final accuracy in our FedSpa framework. \par

\begin{figure*}[!hbtp]
	\centering
	\includegraphics[width=7in]{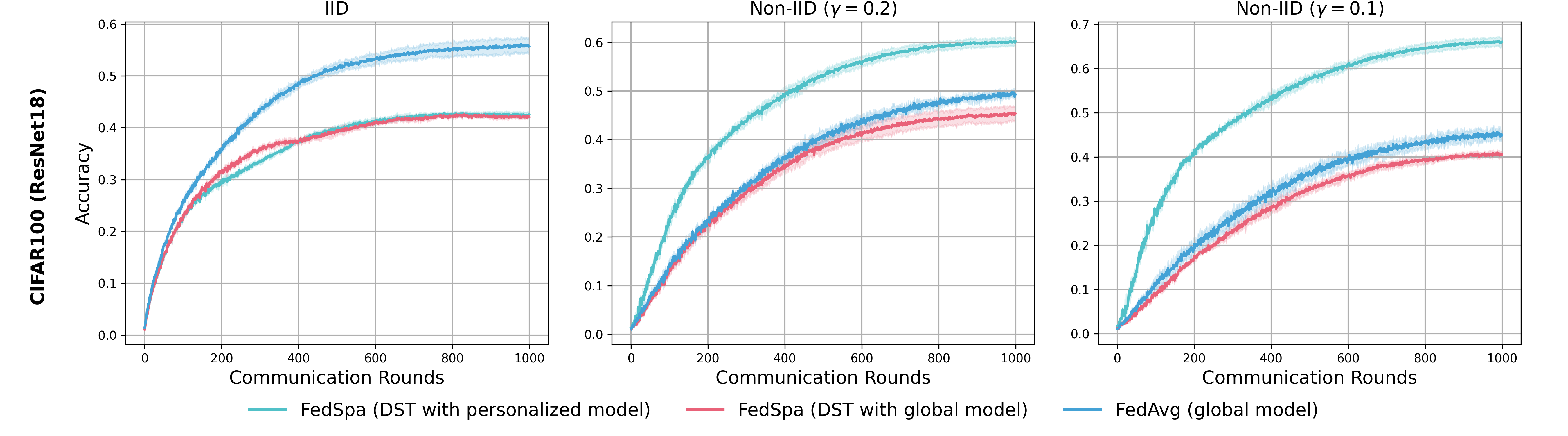}
	\caption{Global model vs. Personalized models. Sparsity of FedSpa is fixed to 0.5.  }
	\label{global ablation}
\end{figure*}
\textbf{Global model vs. Personalized model.} 
In our main experimental result, all the testings are conducted by clients based on their own personalized models. But it is interesting to evaluate whether the global model trained by FedSpa itself could converge, or even could achieve commensurate performance with the global model trained by general FL solution (e.g. FedAvg). As demonstrated by Figure \ref{global ablation}, we empirically find that in the IID setting, the global model trained by FedSpa cannot recover the performance of that trained by FedAvg, and a considerable performance drop is also observed in the Non-IID setting.  Another observation is that the global model of FedSpa surprisingly maintains roughly the same performance as its personalized models in the IID setting, but conceivably suffers significant performance loss in the Non-IID setting. This corroborates our conclusion that sub-networks extracted from a global model may potentially outperform the full model, under the condition that the data distributions of clients are skewed (or heterogeneous).

{\color{black}
	\textbf{Wall time.} Recall that we do an additional mask searching procedure in FedSpa (DST), which might possibly induce extra wall time on the local devices. We show in Table \ref{wall time} the wall time used for training and mask searching in one single local round. The sparsity used for this experiment is fixed to 0.5, while other parameters remain the default setting (see Section \ref{setup}). We use one single 1080Ti to perform training for the GPU-based experiment, while the CPU-based experiment is conducted on an Intel(R) Xeon(R) CPU E5-2620 v4 @ 2.10GHz with 8 cores. Our experimental results confirm that the mask searching process only accounts for a small portion of wall time (approximately $5\%-20\%$) for the entire computation time on the local devices. 
	
	\begin{table*}
		\centering
		\caption{\color{black}Wall time of FedSpa (DST) for local training and mask searching.}
		\begin{tabular}{cccc}
			\midrule 
			\multirow{2}{*} {Task} & Wall Time & Wall Time  & Ratio \\ & (Train) & (Mask Search) &(Mask Search/Train)    \\
			\midrule 
			EMNIST-LeNet (CPU) & 1.03$\pm$0.04s & 0.09$\pm$0.0s & 8.92\%$\pm$0.6\\ 
			CIFAR10-VGG11 (CPU) & 11.4$\pm$0.25s & 2.19$\pm$0.11s & 19.19\%$\pm$1.09\\ 
			CIFAR100-Resnet18 (CPU) & 28.61$\pm$0.38s & 3.7$\pm$0.28s & 12.93\%$\pm$1.08\\ 
			EMNIST-LeNet5 (GPU) & 0.39$\pm$0.01s & 0.02$\pm$0.0s & 5.66\%$\pm$0.25\\ 
			CIFAR10-VGG11 (GPU) & 1.56$\pm$0.03s & 0.22$\pm$0.01s & 14.3\%$\pm$0.48\\ 
			CIFAR100-Resnet18 (GPU) & 2.71$\pm$0.01s & 0.33$\pm$0.01s & 12.06\%$\pm$0.2\\
			\hline
			\label{wall time}
		\end{tabular}
	\end{table*}
}


\section{Conclusions}
In this paper, we propose FedSpa, a personalized FL solution that enables sparse-to-sparse training and efficient sub-model aggregation. As demonstrated by our experiments, FedSpa exhibits outstanding performance in the Non-IID setting, outperforming other existing solutions in terms of accuracy, convergence speed as well as communication overhead. Additionally, we present theoretical analysis to evaluate the convergence bound of FedSpa towards the proposed SPFL problem. Future direction includes designing new model aggregation solutions for the sparse sub-network, and new mask searching techniques specifically targeting on federated learning process.



%

\ifCLASSOPTIONcompsoc
  \section*{Acknowledgments}
\else
  \section*{Acknowledgment}
\fi

This work is supported by Science and Technology Innovation 2030 –“Brain Science and Brain-like Research” Major Project (No. 2021ZD0201402 and No. 2021ZD0201405).

\ifCLASSOPTIONcaptionsoff
  \newpage
\fi




\bibliographystyle{IEEEtran}
\bibliography{bib}

\appendices
\onecolumn
\section{Introduction to main components of FedSpa}
\textbf{ERK initialization.} In Algorithm \ref{algorithm3}, we use Erdós-Rényi Kernel (ERK) originally proposed by \cite{evci2020rigging} to initialize the sparsity of each layer. Specifically, the active parameters of the convolutional layer initialized by ERK is proportional to $1  \frac{n^{1-1}+n^{l}+w^{l}+h^{l}}{n^{l-1} * n^{l} * w^{l} * h^{l}}$, where $n^{l-1}$, $n_l$ $w^{l}$ and $h^{l}$ respectively specify number of input channels, output channels and kernel's width and height in the $l$-th layer. For the linear layer, the number of active parameters scale with $1 \frac{\mathrm{n}^{l-1}+\mathrm{n}^{l}}{\mathrm{n}^{l-1} * \mathrm{n}^{l}} $ where ${n}^{l-1}$ and $\mathrm{n}^{l}$ are the number of neurons in the $l-1$-th and $l$-th layer. This initialization basically allows the layer with less parameters have more proportion of active parameters.
\par
\textbf{Cosine annealing.} Recall that we set the initial pruning rate as $\alpha_0$ and gradually decay it to 0 with cosine annealing \cite{liu2021we}. The update of pruning rate with cosine annealing can be formalized as:  $\alpha_t= 0.5 \times \alpha_0 \times \left(1+\cos \left(\frac{\text { t }}{T-1} \pi\right)\right)$. We perform this decay in order to ensure that the network (specifically, its active coordinates) would not experience drastic change on the later stage of training while ensuring that the mask searching is effective on the early stage of training.

\par

\section{Detailed Experimental Setting}
\label{data splitting}

\subsection{Data Splitting Setting}
\label{Data Splitting Setting}
In our implementation, we first split the training data (60k pieces of data for CIFAR10 and CIFAR100, and 145.6k for EMNIST-L, respectively) to clients for IID setting and Non-IID setting. For the IID setting, data are uniformly sampled for each client. For the Non-IID setting,  we use $\gamma$-Dirichlet distribution on the label ratios to ensure uneven label distributions among devices as \cite{hsu2019measuring}. The lower the distribution parameter $\gamma$ is, the more uneven the label distribution will be, and would be more challenging for FL. After the initial splitting of training data, we sample 100 pieces of testing data from the testing set to each client. To simulate the personalized setting, each client's testing data has the same proportion of labels as its training data. Testing of the personalized model is performed by each client based on their personalized data, and the overall testing accuracy is calculated as the average of all the client's testing accuracy. In our experiment, we simulate different Non-IID settings. For CIFAR10, Non-IID setting A and B respectively specify $\gamma = 0.5$ and $\gamma=0.3$. For EMNIST-L and CIFAR100, since the number of the total labels are bigger\footnote{26 and 100 labels respectively in EMNIST-L and CIFAR100, while only 10 labels in CIFAR10.}, we use smaller $\gamma$, wherein setting A and B respectively specify $\gamma=0.2$ and $\gamma=0.1$. 

\subsection{Network Architectures}
\label{Network Architectures}
We follow the Caffe's implementation of LeNet5 \footnote{Available in https://github.com/mi-lad/snip/blob/master/train.py} \cite{lecun2015lenet}, VGG11  \cite{simonyan2014very} and ResNet18 \cite{he2016deep} to do the evaluation. Suggested by \cite{hsieh2020non}, DNNs with batch normalization layers \cite{ioffe2015batch} are particularly vulnerable to the Non-IID setting, suffering significant model quality loss in the FL process. Following the recommendation from \cite{hsieh2020non}, we use group normalization \cite{wu2018group} to substitute the original batch normalization layer in both ResNet18 and VGG11.

\subsection{Baseline Description}
\label{baseline description}
Below, we give a brief introduction of the baselines compared in our evaluations:
\begin{itemize}[leftmargin=*]
	\item \textbf{FedAvg} \cite{mcmahan2016communication} is the vanilla solution of FL. It utilizes weights average to enable all the clients to collaboratively train a global model, which efficiently absorbs knowledge from personal data resided in clients. 
	
	\item \textbf{Ditto} \cite{li2021ditto} is a personalized FL solution aiming to smooth the tension brought by the data heterogeneity problem of FL. Ditto achieves personalization via maintaining both the local models and global model. Specifically, within each round of iteration, each client first trains the global model based on its local empirical loss (which shares the same procedure as FedAvg). After the global model is updated,  each client additionally trains its local model based on a loss function involving its local empirical loss and the proximal term towards the global model. This local training phase is used to extract the global knowledge into each client's local model. Since each client has to maintain and train both local model and global model, Ditto might need extra computation and storage overhead to achieve its personalization. 
	
	\item \textbf{Local} is the direct solution to the ultimate PFL problem (P2). Each client performs SGD based on its local data, and there is no communication between clients. To mimic the FL setting, we sample 10 out of 100 clients to do the local update on its local model after every 5 epochs of training (same with the number of local epochs in a communication round that is performed by other solutions). For sake of consistency, we still use 1 communication round to represent 5 local epochs of Local in our evaluation.
	
	\item \textbf{Sub-FedAvg} \cite{vahidian2021personalized} is a prominent model compression-based PFL. Sub-FedAvg maintains personalized sub-networks for each client. Training of Sub-FedAvg starts from a fully dense model, and this solution iteratively prunes out the parameters and channels as the training progresses. Finally, the commonly shared parameters of each layer are removed, and only the personalized parameters that can represent the features of local data are kept.
	\item { \color{black} 
		\textbf{Subsampling} \cite{konevcny2016federated}  is a gradient-compression solution aiming to reduce the communication overhead of FL. The local training procedure is the same with FedAvg. The difference is that Subsampling does not communicate the intact model for aggregation, but only communicates the sparse gradient update to the server for aggregation. Explicitly, in each round, the sparse gradient update is produced through element-wisely multiplying a random mask. Different from FedSpa, the randomized mask is independently generated in each round, and would only be used to compress the gradient when uploading the gradient update (which means in the model distribution phase, the model distributed would not be sparsified, and therefore would not save the downlink communication cost). 
	}
\end{itemize}

\section{Missing Proof of Theorem \ref{convergence theorem}}
\label{proof}
In this section, for sake of readability, we first clarify the notations and reiterate several facts that we use in our proof. Then we present several lemmas that are commonly used in the FL literature (see \cite{karimireddy2020scaffold, xu2021fedcm}). Later, several key lemmas are listed with exhaustive proof, and finally, the proof of our main theorem is given by exploiting the listed lemmas and facts. 
\subsection{Notations and Facts}
Throughout the proof, we assume $\sum_{\tau}$ equivalent to  $\sum_{\tau=0}^{N-1}$,  $\sum_{k}$ equivalent to $\sum_{k=1}^{K}$,  and $\sum_{k,\tau}$ equivalent to $\sum_{k=1}^K \sum_{\tau=0}^{N-1}$ unless otherwise specified. In our proof, we reuse most of the notations from our problem formulation part.  We use $\bm g_{k,t,\tau} (\tilde{\bm w}_{k,t,\tau} ) =  \nabla_{\tilde{\bm w}_{k,t,\tau} } \mathcal{L}( \tilde{\bm w}_{k,t,\tau}; \xi_{k,t,\tau})  $ to denote the stochastic gradient of client $k$ in round $t$ and at step $\tau$. 
\par
Then as per our formulation in Section  \ref{sgd solution}, we reiterate the following facts, which would be heavily used in our proof. 
\begin{fact}[Local step]
	\label{local step fact}
	As per Eq. (\ref{local SGD update2}), one local step of client's update can be formalized as follows:
	\begin{equation}
		\begin{split}
			\tilde{\bm w}_{k,t,\tau+1} = \tilde{\bm w}_{k,t,\tau} - \eta_t \bm m_{k,t} \odot \bm g_{k,t,\tau} (\tilde{\bm w}_{k,t,\tau} ),
		\end{split}
	\end{equation}
	where $\bm g_{k,t,\tau} (\tilde{\bm w}_{k,t,\tau} )$ is the stochastic gradient over the sparse model weights $\tilde{\bm w}_{k,t,\tau}$, and $\tilde{\bm w}_{k,t,0}= \bm m_{k,t} \odot \bm w_t$ is the synchronized local weights at the beginning of a communication round.
\end{fact}	
\begin{fact}[Local update from $k$-th client ]
	\label{local update fact}
	The local update of clients can be formalized as follows:
	\begin{equation}
		\begin{split}
			\bm U_{k,t} &= \tilde{\bm w}_{k,t,0} - \tilde{\bm w}_{k,t,N} = \eta_t    \sum_{\tau } \bm m_{k,t} \odot \bm g_{k,t,\tau} (\tilde{\bm w}_{k,t,\tau} ) .
		\end{split}
	\end{equation}
\end{fact}

\begin{fact}[Server's update]
	\label{server aggregation fact}
	The server aggregates the sparse update by averaging, which can be formalized as follows:
	\begin{equation}
		\bm w_{t+1} = \bm w_{t} -  \frac{1}{ S} \sum_{k \in S_t }  \bm U_{k,t}  = \bm w_{t} - \frac{\eta_t }{S} \sum_{k \in S_t,\tau } \bm m_{k,t} \odot \bm g_{k,t,\tau} (\tilde{\bm w}_{k,t,\tau} ) .
	\end{equation}
\end{fact} 
\begin{fact}[Global loss, local loss, and gradient]
	\label{global and local loss}
	As per our SPFL problem (P2), the local loss of a client is denoted by $F_k(\tilde{\bm w}_k)$, and is formulated as:
	\begin{equation}
		F_k( \tilde{\bm w}_k) =\mathbb{E}[ \mathcal{L}_{(\bm x,  y) \sim \mathcal{D}_k }( \tilde{\bm w}_k; (\bm x,  y))]
	\end{equation}
	where $\tilde{\bm w}_k=\bm m_{k}^* \odot \bm w$, and the global loss can be formalized as follows:
	\begin{equation}
		f(\bm w)=\frac{1}{K} \sum_{k=1}^{K}  F_{k}( \tilde{\bm w}_k )=\frac{1}{K} \sum_{k=1}^{K}  F_{k}( \bm m_{k}^* \odot \bm w ).
	\end{equation}
	Finally the gradient of the global loss can be formalized as follows:
	\begin{equation}
		\nabla_w f(\bm w) = \frac{1}{K} \sum_{k=1}^{K} \bm m_{k}^* \odot   \nabla_{\tilde{\bm w}_k} F_{k}(  \tilde{\bm w}_k ),
	\end{equation}
	where $\nabla_{\tilde{\bm w}_k} F_{k}(  \tilde{\bm w}_k )$ is the gradient over the sparse local models.
\end{fact} 
\subsection{Auxiliary Lemmas}
In the following. we shall present several common lemmas that are heavily used in the FL literature.
\begin{lemma}[Cauchy-Schwarz]
	\label{cauchy}
	Assume arbitrary vector sequences $ \{ \bm a_k \}_{k=1,\dots,K}$ and $ \{ \bm b_k \}_{k=1,\dots,K}$, Cauchy-Schwarz inequality implies:  
	\begin{equation}
		\left\|\sum_{k=1}^{K} \bm a_{k} \bm b_{k}\right\|^{2} \leq\left(\sum_{k=1}^{K}\left\|\bm a_{k}\right\|^{2}\right)\left(\sum_{k=1}^{K}\left\| \bm b_{k}\right\|^{2}\right) ,
	\end{equation} 
	by taking $b_k=\mathbf{1}$, we also have:
	\begin{equation}
		\left\|\sum_{k=1}^{K} \bm a_{k} \right\|^{2} \leq K\left(\sum_{k=1}^{K}\left\|\bm a_{k}\right\|^{2}\right) ,
	\end{equation} 
\end{lemma} 

\begin{lemma} [Separating mean and variance, Lemma B.3 \cite{xu2021fedcm}]
	\label {sperating mean}
	Let $\left\{\bm a_{1}, \ldots, \bm a_{\tau}\right\}$ be $\tau$ random vectors in $\mathbb{R}^{d}$ . Suppose that $\left\{\bm a_{i}- \bm \xi_{i}\right\}$ form a martingale difference sequence, i.e. $\mathbb{E}\left[\bm a_{i}- \bm \xi_{i} \mid \bm a_{1}, \ldots ,\bm a_{i-1}\right]=0$, and suppose that their variance is bounded by $\mathbb{E}\left[\left\|\bm a_{i}-\bm \xi_{i}\right\|^{2}\right] \leq \sigma^{2} .$ Then, the following inequality holds:
	$$
	\mathbb{E}\left[\left\|\sum_{i=1}^{\tau} \bm a_{i}\right\|^{2}\right] \leq 2 \left\|\sum_{i=1}^{\tau} \bm \xi_{i}\right\|^{2}+2 \tau \sigma^{2} .
	$$
\end{lemma}
\begin{lemma} [Relaxed triangle inequality, Lemma 3 \cite{karimireddy2020scaffold}] \label{relaxed triangle} Let $\bm v_i$ and $\bm v_j$ be vectors in $\mathbb{R}^{d}$.  Then the following inequality holds true for any $a>0$:
	\begin{equation}
		\left\|\boldsymbol{v}_{i}+\boldsymbol{v}_{j}\right\|^{2} \leq(1+a)\left\|\boldsymbol{v}_{i}\right\|^{2}+\left(1+\frac{1}{a}\right)\left\|\boldsymbol{v}_{j}\right\|^{2}.
	\end{equation}
\end{lemma}
\begin{lemma}
	\label{separate mean 0 variable}
	For random vector  $\bm v_1$ satisfying $\mathbb{E} [\bm v_1]=\left[\begin{array}{c}
		0 \\
		\cdot \\
		\cdot \\
		0
	\end{array}\right]$, and assume another random vector $\bm v_2$ is independent with $\bm v_1$, we have:
	\begin{equation}
		\mathbb{E}	[\left\|\bm v_1 + \bm v_2 \right\|^{2} ] = \mathbb{E} [\| \bm v_1 \|^2] + \mathbb{E} \| \bm v_2 \|^2 
	\end{equation}
\end{lemma}
\begin{proof}
	\begin{equation}
		\begin{split}
			\mathbb{E}	[\left\|\bm v_1 + \bm v_2\right\|^{2} ] &= \mathbb{E} \langle \bm v_1 + \bm v_2, \bm v_1+ \bm v_2  \rangle \\
			& = \mathbb{E} \| \bm v_1 \|^2 + \mathbb{E} \| \bm v_2 \|^2 + 2 \mathbb{E}\langle \bm v_1 ,\bm v_2 \rangle \\
			& = \mathbb{E} \| \bm v_1 \|^2 + \mathbb{E} \| \bm v_2 \|^2 + 2 \langle \mathbb{E}\bm v_1 ,\mathbb{E} \bm v_2 \rangle \\
			& = \mathbb{E} \| \bm v_1 \|^2 + \mathbb{E}\| \bm v_2\|^2 
		\end{split}
	\end{equation}
	This completes the proof. 
\end{proof}

\subsection{Key Lemmas}
In this section, we present several important lemmas that would be used in our formal proof. All the presented claims are rigorously proved. 
\begin{lemma}[Smoothness of  $f ( \bm w)$] \label{Smoothness of f} Assume $\tilde{\bm w}_k = \bm m_k^* \odot \bm w $ for any $\bm m_k^* \in \{0,1\}^d $, we have $ L$-smoothness for $ f ( \bm w) =\frac{1}{K} \sum_k F_k (\tilde{\bm w}_k)$, i.e., for any $\bm w_1, \bm w_2 \in \mathbb{R}^d$, we have:
\end{lemma}
\begin{equation}
	\begin{split}
		\| \nabla f(\bm w_1) - \nabla f(\bm w_2) \|  \leq   L  \|\bm w_{1}-\bm w_{2} \| .
	\end{split}
\end{equation}
\begin{proof}
	\begin{equation}
		\begin{split}
			\| \nabla f(\bm w_1) - \nabla f(\bm w_2) \| & =  \left \| \frac{1}{K} \sum_k (\bm m_{k}^* \odot \nabla F_k ( \bm m_k^* \odot \bm w_1)  - \bm m_{k}^*  \odot \nabla F_k ( \bm m_{k}^* \odot \bm w_2) ) \right \|   \\
			& \leq    \frac{1}{K}  \sum_k  \| \bm m_{k}^* \odot ( \nabla F_k (\bm m_k^* \odot \bm w_1)  - \nabla F_k (\bm m_k^* \odot \bm w_2 ) ) \|\\
			& \leq     \frac{1}{K}  \sum_k  \|\nabla F_k (\bm m_k^* \odot \bm w_1)  - \nabla F_k (\bm m_k^* \odot \bm w_2 ) \| \\
			& \overset{(a)}{\leq}      \frac{L}{K}  \sum_k  \|\bm m_k^* \odot (\bm w_1- \bm w_2)  \| \\
			& \leq   L   \|\bm w_{1}-\bm w_{2} \|
		\end{split}
	\end{equation}
	where the first equality holds by the Fact \ref{global and local loss} and inequality (a) is due to Assumption \ref{L-smoothness}. This completes the proof. 
\end{proof}
\begin{lemma}[Separating mean and variance of stochastic gradient] For $\bm m_{k,t} \in \{ 0,1\}^d$,  We theoretically prove the following upper-bound for the expected average gradient:  
	\label{seperate variance second lemma}
	\begin{equation}
		\mathbb{E}_t\left[   \left\| \frac{1}{S} \sum_{ k \in S_ {t}, \tau} \bm m_{k,t} \odot \bm g_{k,t,\tau}(\tilde {\bm w}_{k,t,\tau})    \right\|^{2} \right]  \leq  2  \mathbb{E}_t\left[    \left\| \frac{1}{S} \sum_{ k \in S_ {t}, \tau}  \bm m_{k,t} \odot \nabla F_k( \tilde{ \bm w}_{k,t,\tau})    \right\|^{2} \right]  +  \frac{2N   \sigma^2}{S}. \\
	\end{equation} 
	where $\mathbb{E}_t[\cdot]$ denotes the expectation over all the randomness of round $t$.
\end{lemma}
\begin{proof} 
	We view $\frac{1}{S} \bm m_{k,t} \odot \bm g_{k,t,\tau}(\tilde{\bm w}_{k,t,\tau})$ for all $k$ and $\tau$ as stochastic vectors. By the unbiasedness of stochastic gradient, we know their variance satisfies:
	\begin{equation}
		\begin{split}
			&||\frac{1}{S} \bm m_{k,t} \odot \bm g_{k,t, \tau}(\tilde{\bm w}_{k,t,\tau})- \mathbb{E}[ \frac{1}{S}  \bm m_{k,t} \odot \bm g_{k,t, \tau}(\tilde{\bm w}_{k,t,\tau}) ||^2] \\
			=&||\frac{1}{S} \bm m_{k,t} \odot (\bm g_{k,t, \tau}(\tilde{\bm w}_{k,t,\tau})-  \nabla F_k (\tilde{\bm w}_{k,t,\tau}) ) ||^2 \\
			\leq& \frac{ \sigma^2}{S^2}   \\ 
		\end{split}
	\end{equation} 
	where  the last inequality is due to Assumption \ref{bounded variance}.\par
	As per the variance given above, and  directly apply  Lemma \ref{sperating mean}, the claim immediately shows. 
\end{proof}

\begin{lemma}[Drift towards Synchronized Point]
	\label{Drift towards Synchronized Point}
	For any $t \in \{1,\dots,T\}$, $\tau \in \{0,\dots,N \}$, and learning rate satisfies $\eta_t = \sqrt{\frac{1}{16 L^2 N^2p_t T }}$, we have the following claim: 
	\begin{equation}
		\begin{split}
			\frac{1}{K}\sum_{ k }  \mathbb{E}_t \left[ \|  \tilde{ \bm w}_{k,t,\tau} -   \tilde{\bm {w}}_{k,t}   \|^2  \right ] \leq &   5N\eta_t^2   (\sigma^2 + 6N  G^2 ) + 30N^2 \eta_t^2 B^2
		\end{split}
	\end{equation} 
	where  $\mathbb{E}_t[\cdot]$ denotes the expectation over all the randomness of round $t$.
\end{lemma}
\begin{proof}
	We follow the basic techniques from \cite{reddi2020adaptive} to prove this lemma. We first assume that:
	\begin{itemize}
		\item $O=  \bm m_{k,t} \odot  \bm g_{k,t,\tau-1}(\tilde{\bm w}_{k,t,\tau-1} )   - \bm m_{k,t} \odot \nabla F_k( \tilde{\bm w}_{k,t,\tau-1})$
		\item $P= \bm m_{k,t} \odot \nabla F_k( \tilde{\bm w}_{k,t, \tau-1}) -  \bm m_{k,t} \odot \nabla F_k( \tilde{\bm w}_{k,t})$
		\item $Q=  \bm m_{k,t} \odot \nabla F_k( \tilde{\bm w}_{k,t}) -  \nabla f( \bm w_{t})$
	\end{itemize}
	Then we  expand $\tilde{\bm w}_{k,t,\tau}$ as follows:
	\begin{equation}
		\label{one round progress initial} 
		\begin{split}
			& \frac{1}{K}\sum_{ k } \mathbb{E}_t \left[ \|   \tilde{\bm w}_{k,t,\tau} -   \tilde{\bm {w}}_{t}   \|^2  \right ]\\	 
			\overset{\text{Fact \ref{local step fact}}}= 	&  \frac{1}{K}\sum_{ k }  \mathbb{E}_t  \left[ \|   \tilde{\bm w}_{k,t,\tau-1} -   \tilde{\bm {w}} _{t} - \eta_t \bm m_{k,t}  \odot  \bm g_{k,t,\tau-1} (\tilde{\bm w}_{k,t,\tau-1} )    \|^2  \right ] \\
			= 	&  \frac{1}{K}\sum_{ k }  \mathbb{E}_t  \left[ \|   \tilde{\bm w}_{k,t,\tau-1} -    \tilde{\bm {w}}_{t}  - \eta_t    (O + P  + Q  + \nabla f( \bm w_{t}) ) ||^2 \right ] \\
			\leq 	&     \frac{1+ \frac{1}{2N-1}}{K}\sum_{ k } \mathbb{E}_t  \|  \tilde{ \bm w}_{k,t,\tau-1} -   \tilde{\bm {w}}_{t} \|^2  +   \frac{\eta_t^2 }{K}\sum_{ k }  \mathbb{E}_t  \|O\|^2 +  \frac{6N \eta_t^2 }{K}\sum_{ k } \mathbb{E}_t  \| P\|^2   \\ 
			& +  \frac{6N\eta_t^2 }{K}\sum_{ k } \mathbb{E}_t  \|Q\|^2 +  \frac{6N\eta_t^2 }{K}\sum_{ k }  \|\nabla f( \bm w_{t})\|^2  
		\end{split}
	\end{equation} 
	where the last inequality follows from Lemma  \ref{relaxed triangle} and Lemma \ref{separate mean 0 variable}. Explicitly, we use Lemma \ref{separate mean 0 variable} to treat the stochastic term with $\mathbb{E}_t  \|O\|^2$, and then we use Lemma \ref{relaxed triangle} with $a=2N-1$ to separate the other four terms.   \par
	Then we proceed by separately bounding the components in the above inequality.\par
	\textbf{Bounding the second term:}
	\begin{equation}
		\begin{split}
			\frac{\eta_t^2  }{K} \sum_{ k } \mathbb{E}_t 	\| O\|^2  &= \frac{\eta_t^2  }{K} \sum_{ k }  \mathbb{E}_t    \| \bm m_{k,t}   \odot \bm g_{k,t,\tau-1}(\tilde{\bm w}_{k,t,\tau-1} )   - \bm m_{k,t} \odot \nabla F_k( \tilde{\bm w}_{k,t,\tau-1}) \|^2 \\ 
			&= \frac{\eta_t^2 }{K} \sum_{ k }  \mathbb{E}_t  \|   \bm m_{k,t} \odot( \bm g_{k,t,\tau-1}(\tilde{\bm w}_{k,t,\tau-1} )   -  \nabla F_k( \tilde{\bm w}_{k,t,\tau-1}) ) \|^2   \\
			& \leq \frac{\eta_t^2}{K} \sum_{ k }  \mathbb{E}_t  \|    \bm g_{k,t,\tau-1}(\tilde{\bm w}_{k,t,\tau-1} )   -  \nabla F_k( \tilde{\bm w}_{k,t,\tau-1}) \|^2   \\
			&\leq \eta_t^2 \sigma^2  
		\end{split}
	\end{equation}
	where  the last inequality holds by Assumption \ref{bounded variance}. \par 
	\textbf{Bounding the third term:} 
	\begin{equation}
		\begin{split}
			\frac{6N\eta_t^2 }{K}\sum_{ k } \mathbb{E}_t 	\| P\|^2 
			=& 	\frac{6N\eta_t^2 }{K}\sum_{ k } \mathbb{E}_t  \| \bm m_{k,t} \odot \nabla F_k( \tilde{\bm w}_{k,t, \tau-1}) -  \bm m_{k}^* \odot \nabla F_k( \tilde{\bm w}_{k,t}) \|^2 \\
			= & 	\frac{6N\eta_t^2 p_t  }{K}\sum_{ k }  \mathbb{E}_t  \|  \nabla F_k( \tilde{\bm w}_{k,t,\tau-1}) -  \nabla F_k( \tilde{\bm w}_{k,t}) \|^2 \\
			\leq & 	(6N\eta_t^2  p_t L^2) \frac{1}{K} \sum_k \mathbb{E}_t  \| \tilde{\bm w}_{k,t,\tau-1} -\bm \tilde{\bm w}_{k,t}   \|^2 \\
		\end{split}
	\end{equation}
	where the last equation holds by the definition of $p_t$.
	\par
	\textbf{Bounding the fourth term:}
	\begin{equation}
		\begin{split}
			&\frac{6N\eta_t^2 }{K}\sum_{ k } \mathbb{E}_t 	\| Q\|^2 \\
			=& \frac{6N\eta_t^2 }{K}\sum_{ k }  \mathbb{E}_t  \| \bm m_{k}^* \odot \nabla F_k( \tilde{\bm w}_{k,t}) -    \frac{1}{K}\sum_{k^{\prime}} \bm m_{k^{\prime}}^* \odot \nabla F_k( \tilde{\bm w}_{k,t}) \|^2 \\
			\leq & 6N \eta_t^2 G^2
		\end{split}
	\end{equation}
	where the last inequality holds by Assumption \ref{bounded gradient dissimilarity}.  \par
	\textbf{Putting together:} Plugging all the components into Eq.(\ref{one round progress initial}) , the following result immediately follows:
	\begin{equation}
		\begin{split}
			&\frac{1}{K}\sum_{ k } \mathbb{E}_t  \left[ \|   \tilde{\bm w}_{k,t,\tau} -   \tilde{\bm {w}}_{k,t}   \|^2  \right ]\\
			\leq & (1+ \frac{1}{2N-1}+ 6N \eta_t^2  p_t L^2) \frac{1}{K} \sum_{k}  \mathbb{E}_t  \|   \tilde{\bm w}_{k,t,\tau-1} -   \tilde{\bm {w}}_{k,t} \|^2  + \eta_t^2  \sigma^2    + 6N  \eta_t^2 G^2   + 6N\eta_t^2   \|\nabla f( \bm w_{t})\|^2 \\   
			\leq&  (1 + \frac{1}{N-1} )  \frac{1}{K} \sum_{k}  \mathbb{E}_t  \|   \tilde{\bm w}_{k,t,\tau-1} -   \tilde{\bm {w}}_{k,t} \|^2  + \eta_t^2  ( \sigma^2 + 6N  G^2) + 6N\eta_t^2   \|\nabla f( \bm w_{t})\|^2, \\
		\end{split}
	\end{equation}
	where the last inequality holds by our assumption $\eta_t \leq  \sqrt{\frac{1}{16 L^2 N^2 p_t T}} $. \par
	\textbf{Unrolling the recursion}, we obtain the following results:
	\begin{equation}
		\begin{split}
			&\frac{1}{K}\sum_{ k } \mathbb{E}_t  \left[ \|   \tilde{\bm w}_{k,t,\tau} -   \tilde{\bm {w}}_{k,t}   \|^2  \right ]\\
			\leq & \sum_{\tau=0}^{N-1} (1+ \frac{1}{N-1})^{\tau} \left [ \eta_t^2 ( \sigma^2 + 6N G^2 ) + 6N\eta_t^2   \|\nabla f( \bm w_{t})\|^2  \right] \\
			\leq & (N-1) \times \left ( (1+ \frac{1}{N-1})^{N}-1\right ) \left [ \eta_t^2   (\sigma^2  + 6N G^2 ) + 6N\eta_t^2   \|\nabla f( \bm w_{t})\|^2  \right] \\
			\leq &   5N\eta_t^2   (\sigma^2 + 6N  G^2 ) + 30N^2 \eta_t^2 B^2\\
		\end{split}
	\end{equation}
	The second last inequality holds since $ \left ( (1+ \frac{1}{N-1})^{N}-1\right ) \leq 5$ for $N \geq 1$, and the last inequality holds by Assumption \ref{bounded gradient}. This completes the proof. 
\end{proof}

\subsection{Formal Proof }
We start our proof by expanding $f(\bm w_{t+1})$ under its smoothness condition (see Lemma \ref{Smoothness of f}), which indicates that:
\begin{equation}
	\begin{split}
		\label{entrance}
		&\mathbb{E}_t \left[f \left(  {{\bm w}}_{t+1} \right) \mid \bm {w}_{t} \right] \\
		\leq & f (  {\bm {w}}_{t}) -   \left \langle \nabla f ( {\bm w}_{t} ), \mathbb{E}_t [ {\bm {w}}_{t+1}-  {\bm {w}}_{t} ] \right \rangle  + \frac{ L}{2} \mathbb{E}_t ||  {\bm {w}}_{t+1}-  {\bm {w}}_{t}   ||^2  \\
		\overset{\text {Fact \ref{server aggregation fact}}} = &   f \left(   {\bm {w}}_{t}\right)-   \eta_t   \mathbb{E}_t \left [ \left\langle \nabla f \left(   {\bm {w}}_{t} \right), \frac{1} {S}   \sum_{ k \in S_t }   \bm U_{k,t} \right\rangle \right ]+  \frac{ L}{2} \mathbb{E}_t\left\| \frac{1} {S} \sum_{ k \in S_ {t}}   \bm U_{k,t} \right\|^{2} \\
		\overset{\text {Fact \ref{local update fact}}} = &   f \left(   {\bm {w}}_{t}\right)-  \frac{\eta_t}{N}    \left\langle N\nabla f \left(   {\bm {w}}_{t} \right) , \mathbb{E}_t \left [ \frac{1} {K}   \sum_{ k, \tau }  \bm m_{k,t} \odot \nabla F_k( \tilde{\bm w}_{k,t,\tau}) \right ] \right\rangle  +  \frac{  L}{2} \mathbb{E}_t \left\| \frac{1} {S} \sum_{ k \in S_ {t}}   \bm U_{k,t} \right\|^{2} \\
		\leq &   f \left(   {\bm {w}}_{t}\right)-  \frac{\eta_t N}{2} ||    \nabla f \left(   {\bm {w}}_{t} \right) ||^2 + \underbrace{ \frac{\eta_t }{2N} \mathbb{E}_t || \frac{1} {K}   \sum_{ k, \tau }  \bm m_{k,t} \odot \nabla F_k( \tilde{\bm w}_{k,t,\tau}) - N \nabla f \left(   {\bm {w}}_{t} \right) ||^2}_{T_1}  \\  & \qquad \qquad \qquad \qquad \qquad \qquad \qquad \qquad \qquad \qquad \qquad +  \underbrace{ \frac{ L}{2} \mathbb{E}_t \left\| \frac{1} {S} \sum_{ k \in S_ {t}}   \bm U_{k,t} \right\|^{2} }_{T_2} \\
	\end{split}
\end{equation} 
where the last inequality holds since $-ab \leq \frac{1}{2} ((b-a)^2-a^2)$, and $\mathbb{E}_t[\cdot]$ is the expectation over all the randomness in round $t$. \par
In the following, we shall separately bound $T_1$ and $T_2$.  \par
\textbf{Bounding $T_1$:} 
\begin{equation}
	\begin{split}
		T_1=  &     \frac{\eta_t }{2N} \mathbb{E}_t  || \frac{1} {K}   \sum_{ k, \tau }  \bm m_{k,t} \odot \nabla F_k( \tilde{\bm w}_{k,t,\tau}) - N \nabla f \left(   {\bm {w}}_{t} \right) ||^2\\
		\overset{\text{Fact } \ref{global and local loss}}=&     \frac{\eta_t }{2N} \mathbb{E}_t  || \frac{1} {K}   \sum_{ k, \tau }  \bm m_{k,t} \odot \nabla F_k( \tilde{\bm w}_{k,t,\tau}) - N \frac{1}{K} \sum_k   \bm m_{k}^* \odot \nabla F_k (  \tilde{\bm w}_{k,t}) ||^2\\
		= &     \frac{\eta_t }{2N} \mathbb{E}_t  || \frac{1} {K}   \sum_{ k, \tau }  ( \bm m_{k,t} \odot \nabla F_k( \tilde{\bm w}_{k,t,\tau}) -     \bm m_{k}^* \odot \nabla F_k (  \tilde{\bm w}_{k,t}) ) ||^2\\
		\overset{(a)}\leq &     \frac{\eta_t }{2K}  \sum_{ k, \tau }  \mathbb{E}_t  || \bm m_{k,t} \odot \nabla F_k( \tilde{\bm w}_{k,t,\tau}) -    \bm m_{k}^* \odot \nabla F_k (  \tilde{\bm w}_{k,t})  ||^2\\
		\overset{(b)}= & \frac{\eta_t p_{t}  }{2K}  \sum_{ k, \tau } \mathbb{E}_t  \|\nabla F_k( \tilde{\bm w}_{k,t,\tau}) -\nabla F_k( \tilde{\bm w}_{k,t}) \|^2   \\
		\leq & \frac{\eta_t  L^2 p_t }{2K}  \sum_{k, \tau } \mathbb{E}_t [  ||   \tilde{\bm w}_{k,t,\tau} -     \tilde{\bm w}_{k,t}  ||^2 ]\\
	\end{split}
\end{equation} 
where inequality (a) is due to Cauchy-Schwarz inequality (i.e., Lemma \ref{cauchy}), (b) follows from the definition $p_{t} = \max_{t \in [k], n \in [N]} \left \{ \frac{ || \bm m_{k,t} \odot \nabla F_k( \tilde{\bm w}_{k,t,\tau}) -    \bm m_{k}^* \odot \nabla F_k (  \tilde{\bm w}_{k,t})  ||^2}{\|\nabla F_k( \tilde{\bm w}_{k,t,\tau}) -\nabla F_k( \tilde{\bm w}_{k,t}) \|^2} \right \}$. The last inequality holds by Assumption \ref{L-smoothness}.
\par
Plugging the results of Lemma \ref{Drift towards Synchronized Point}, we obtain that:
\begin{equation}
	\begin{split}
		T_1 & \leq \frac{\eta_t  L^2 p_t N }{2} (5N\eta_t^2   (\sigma^2 + 6N  G^2 ) + 30N^2 \eta_t^2 B^2)  \\
		&\leq \frac{5N^2 \eta_t^3  L^2 p_t }{2}  (\sigma^2 + 6N G^2 ) +  15 N^3\eta_t^3  L^2  B^2 p_t
	\end{split}
\end{equation} 
\textbf{Bounding $T_2$:} 
\begin{equation}
	\begin{split}	
		\label{the third term}
		T_2
		= & \frac{ L}{2} \mathbb{E}_t \left\| \frac{1} {S} \sum_{ k \in S_ {t}}   \bm U_{k,t} \right\|^{2} \\
		\overset{\text{Fact \ref{local update fact}}}= & \frac{ L \eta_t^2}{2} \mathbb{E}_t\left[    \left\| \frac{1}{S} \sum_{ k \in S_ {t}, \tau} \bm m_{k,t} \odot \bm g_{k,t,\tau}(\bm \tilde{\bm w}_{k,t,\tau})    \right\|^{2} \right] \\
		\overset{(a)}  \leq  & L\eta_t^2  \mathbb{E}_t\left[    \left\| \frac{1}{S} \sum_{ k \in S_ {t}, \tau} \bm m_{k,t} \odot \nabla F_k(\bm \tilde{\bm w}_{k,t,\tau})    \right\|^{2} \right] + \frac{N   L  \eta_t^2 \sigma^2}{S}\\
		\leq  &   L\eta_t^2   N \sum_\tau \underbrace{\mathbb{E}_t\left[\left\| \frac{1}{S} \sum_{ k \in S_ {t}} \bm m_{k,t} \odot \nabla F_k(\bm \tilde{\bm w}_{k,t,\tau})    \right\|^{2} \right]}_{ T_{3}} + \frac{N  L  \eta_t^2 \sigma^2}{S}\\ 
	\end{split}
\end{equation} 
where (a) is obtained as per Lemma \ref{seperate variance second lemma}. \par
\textbf{Bounding $ T_{3}$:} 
\begin{equation}
	\begin{split}
		& T_{3} \\
		= &\frac{1}{S^2}  \mathbb{E}_t  \left \langle  \sum_{ i \in [K] } \mathbb{I}_{\{ i \in S_t \}}\bm m_{i,t} \odot \nabla F_{i}( \tilde{\bm w}_{i,t,\tau}) ,    \sum_{ j \in [K]} \mathbb{I}_{\{ j \in S_t \}}\bm m_{{j},t} \odot \nabla F_{j}( \tilde{\bm w}_{j,t,\tau}) \right \rangle \\
		= &\frac{1}{S^2} \mathbb{E}_t \left [  \sum_{ i , j \in [K], j \neq i,\tau}  \mathbb{E}_{S_t}[\mathbb{I}_{\{ i \in S_t \cap j \in S_t \}} ] \left \langle  \bm m_{i,t} \odot \nabla F_{i}( \tilde{\bm w}_{i,t,\tau}) ,   \bm m_{{j},t} \odot \nabla F_{j}( \tilde{\bm w}_{j,t,\tau}) \right \rangle \right. \\  
		& \qquad \qquad \left . + \sum_{ i } \mathbb{E}_{S_t}[ \mathbb{I}_{\{ i \in S_t \}} ]\| \bm m_{{i},t} \odot \nabla F_{i}( \tilde{\bm w}_{i,t,\tau})  \|^2  \right]  \\
		= &\frac{1}{S^2} \mathbb{E}_t \left [  \sum_{ i , j \in [K], j \neq i}  \frac{S(S-1)}{K(K-1)} \left \langle  \bm m_{i,t} \odot \nabla F_{i}( \tilde{\bm w}_{i,t,\tau}) ,   \bm m_{{j},t} \odot \nabla F_{j}( \tilde{\bm w}_{j,t,\tau}) \right \rangle + \sum_{ i } \frac{S}{K}\| \bm m_{{i},t} \odot \nabla F_{i}( \tilde{\bm w}_{i,t,\tau})  \|^2  \right]  \\
		= &\frac{1}{S^2} \mathbb{E}_t \left [  \sum_{ i , j \in [K]}  \frac{S(S-1)}{K(K-1)} \left \langle  \bm m_{i,t} \odot \nabla F_{i}( \tilde{\bm w}_{i,t,\tau}) ,   \bm m_{{j},t} \odot \nabla F_{j}( \tilde{\bm w}_{j,t,\tau}) \right \rangle  \right . \\ 
		&\left . \qquad \qquad \qquad + \sum_{ i } \frac{S(K-S)}{K(K-1)}\| \bm m_{{i},t} \odot \nabla F_{i}( \tilde{\bm w}_{i,t,\tau})  \|^2  \right]  \\
		\leq  & \mathbb{E}_t \left [  \underbrace{  \frac{1}{K^2}\| \sum_{ k} \bm m_{k,t} \odot \nabla F_k( \tilde{\bm w}_{k,t,\tau})  \|^2 }_{T_{4}} + \underbrace{\frac{(K-S)}{S K(K-1)} \sum_{ k } \| \bm m_{{k},t} \odot \nabla F_{k}( \tilde{\bm w}_{k,t,\tau})  \|^2}_{T_{5}}  \right]  \\
	\end{split}	
\end{equation} 
where the last inequality holds since $\frac{S-1}{S K(K-1)} =\frac{S-1}{S K^2-SK} \leq\frac{S}{S K^2} = \frac{1}{K^2}   $. \par
\textbf{Bounding $T_4$}:
\begin{equation}
	\begin{split}
		T_{4}  = &    \left\|   (\frac{1}{K} \sum_{ k  } \bm m_{k,t} \odot \nabla F_k( \tilde{\bm w}_{k,t,\tau})-  \nabla f(  \bm w_{t}) ) +  \nabla f(  \bm w_{t}) )  \right\|^{2} \\
		\overset{\text{Lemma \ref{cauchy}}}\leq & 2\left\|   \frac{1}{K} \sum_{ k  } \bm m_{k,t} \odot \nabla F_k( \tilde{\bm w}_{k,t,\tau})- \nabla f(  \bm w_{t})  \right\|^{2}+ 2 \left\|   \nabla f(  \bm w_{t})  \right \|^2 \\
		\overset{\text{Fact \ref{local update fact}}}= & 2\left\|   \frac{1}{K} \sum_{ k  } (\bm m_{k,t} \odot \nabla F_k( \tilde{\bm w}_{k,t,\tau})-  \bm m_{k}^* \odot \nabla  F_k(  \tilde{\bm w}_{k,t}) ) \right\|^{2}+ 2  \left\|  \nabla f(  \bm w_{t})  \right \|^2 \\
		\overset{\text{Lemma \ref{cauchy}}}\leq  &   \frac{2}{K} \sum_{ k}    \left\| \bm m_{k,t} \odot \nabla F_k( \tilde{\bm w}_{k,t,\tau})-  \bm m_{k}^* \bm \nabla  F_k(  \tilde{\bm w}_{k,t})  \right\|^{2}+ 2  \left\|  \nabla f(  \bm w_{t})  \right \|^2 \\
		=   &   \frac{2 p_t }{K} \sum_{ k}  \|\nabla F_k( \tilde{\bm w}_{k,t,\tau}) -\nabla F_k( \tilde{\bm w}_{k,t}) \|^2 + 2 \left\|  \nabla f(  \bm w_{t})  \right \|^2 \\
		\leq  &   \frac{2 p_t L^2 }{K} \sum_{ k} \|\tilde{\bm w}_{k,t,\tau} - \tilde{\bm w}_{k,t} \|^2 + 2 \left\|  \nabla f(  \bm w_{t})  \right \|^2 \\
	\end{split}	
\end{equation} 
\par
where the last equation is obtained by the definition of $p_{t}$ and the last inequality holds by L-smoothness. \par
\textbf{Bounding $T_5$}:
\begin{equation}
	\begin{split}
		T_{5}=&\frac{(K-S)}{S K(K-1)} \sum_{ k } \| \bm m_{{k},t} \odot \nabla F_{k}( \tilde{\bm w}_{k,t,\tau} )\|^2   \\
		\leq &\frac{3(K-S)}{S K(K-1)} \sum_{ k } (\| \bm m_{{k},t} \odot \nabla F_{k}( \tilde{\bm w}_{k,t,\tau}) -\bm m_{{k}}^* \odot \nabla F_{k}( \tilde{\bm w}_{k,t}) \|^2  \\
		& \qquad \qquad \qquad \qquad \qquad + \| \bm m_{{k}}^* \odot \nabla F_{k}( \tilde{\bm w}_{k,t}) -\nabla f(  \bm w_{t}) \|^2 +  \| \nabla f(  \bm w_{t})  \|^2 )   \\
		\overset{(a)}\leq&  \frac{3 L^2(K-S) p_t }{SK(K-1)} \sum_{k} \|\tilde{\bm w}_{k,t,\tau}- \tilde{\bm w}_{k,t} \|^2 +  \frac{3 (K-S)G^2}{S(K-1)}   \\
		&  \quad  +  \frac{3 (K-S)}{S(K-1)} \| \nabla f(  \bm w_{t}) \|^2 \\
		\leq& \frac{3 L^2  p_t }{K} \sum_{k}  \|\tilde{\bm w}_{k,t,\tau}- \tilde{\bm w}_{k,t} \|^2 + 3 G^2 +  3\| \nabla f(  \bm w_{t}) \|^2,
	\end{split}
\end{equation}
where the last inequality holds since $\frac{K-S}{S(K-1)} \leq 1 $ under the condition $S\geq 1$. Inequality (a)  is obtained by Assumption \ref{L-smoothness} and the definition of $p_t$.  \par
\textbf{Summing $T_4$ and $T_5$}, we have the following bounding for $T_3$:
\begin{equation}
	\begin{split}
		T_3 &\leq \frac{5 L^2 p_t }{K} \sum_{k}  \|\tilde{\bm w}_{k,t,\tau}- \tilde{\bm w}_{k,t} \| +    3 G^2  +  5  \| \nabla f(  \bm w_{t}) \|^2 \\
		&  \leq \frac{5 L^2 p_t}{K} \sum_{k}  \|\tilde{\bm w}_{k,t,\tau}- \tilde{\bm w}_{k,t} \| +   3 G^2 +  5  B^2 
	\end{split}
\end{equation}
Plugging Lemma \ref{Drift towards Synchronized Point} into the above inequality, we have:
\begin{equation}
	\begin{split}
		&T_{3} \\
		\leq &  5  L^2 p_t \left ( 5N\eta_t^2   (\sigma^2 + 6N  G^2 ) + 30N^2 \eta_t^2 B^2   \right )+3 G^2 +  5  B^2   \\
		= &  25 \eta_t^2  L^2 N p_t   (\sigma^2 + 6N G^2) + 150N^2\eta_t^2  L^2 p_t B^2 + 3G^2+5 B^2 
	\end{split}
\end{equation}
\textbf{Plugging $T_3$ into Inequality (\ref{the third term})}, we bound $T_2$ as follows:
\begin{equation}
	\begin{split}
		T_2 \leq  &    150N^4\eta_t^4  L^3 p_t  B^2+  25 \eta_t^4  L^3 N^3 p_t   (\sigma^2 + 6NG^2)  \\
		& \qquad \qquad \qquad  + 3\eta_t^2 N^2  L G^2  +5 L \eta_t^2   N^2 B^2 + \frac{N  L  \eta_t^2 \sigma^2}{S} .\\ 
	\end{split}
\end{equation}
\textbf{ Plugging $T_2$ and $T_1$ into R.H.S of Inequality (\ref{entrance})}, we obtain that:
\begin{equation}
	\begin{split}
		&\mathbb{E}\left[f \left(  {{\bm w}}_{t+1} \mid \bm {w}_{t} \right) \right] \leq f \left(   {\bm {w}}_{t}\right)-   \frac{\eta_t N}{2} ||    \nabla f \left(   {\bm {w}}_{t} \right) ||^2 \\
		+  & \frac{\eta_t^2 N^2 L p_t }{2} \left ( (50 \eta_t^2 L^2 N+5\eta_t  L  )    (\sigma^2 +6NG^2) \right .\\
		+ & \left.  (300N^2\eta_t^2  L^2 + 30 N\eta_t  L  )  B^2 + \frac{1}{p_t} (6 G^2  + \frac{2    \sigma^2}{NS}+10B^2  ) \right)  .
	\end{split}
\end{equation}
Taking expectation over the randomness before round $t$ towards both sides of the inequality, it yields:
\begin{equation}
	\begin{split}
		\mathbb{E} [||    \nabla f \left(   {\bm {w}}_{t} \right) ||^2] \leq & \frac{2(\mathbb{E}\left[ f \left(  {{\bm w}}_{t} \right) \right ] )- \mathbb{E}\left[ f \left(  {{\bm w}}_{t+1}   \right)\right] }{\eta_t N } 	+   \eta_t N L p_t  \left ( (50 \eta_t^2 L^2 N+5\eta_t  L  )    (\sigma^2 +6NG^2) \right .\\
		+ & \left.  (300N^2\eta_t^2  L^2 + 30 N\eta_t  L  )  B^2 + \frac{1}{p_t} (6 G^2  + \frac{2    \sigma^2}{NS}+10B^2  ) \right)
	\end{split}
\end{equation}
Plugging $\eta_t \leq \sqrt{\frac{1}{16 L^2 N^2 p_t T }} $ into the above inequality, we have:
\begin{equation}
\label{near final}
	\begin{split}
		\mathbb{E} [||    \nabla f \left(   {\bm {w}}_{t} \right) ||^2] \leq \frac{\sqrt{p_t} V_t }{ \sqrt{T} } +  \frac{C_1}{\sqrt{ Tp_t}} +\frac{C_2}{\sqrt{T}}
	\end{split}
\end{equation}
where 
$V_t= 8L (\mathbb{E}\left[ f \left(  {{\bm w}}_{t} \right) \right ] - \mathbb{E}\left[ f \left(  {{\bm w}}_{t+1}   \right)\right])$, $ C_1 = \frac{25}{32NT} (\sigma^2+6NG^2)+\frac{75B^2}{16T} + \frac{3G^2}{2}   + \frac{   \sigma^2}{2NS}+\frac{5B^2}{2}   $ and $C_2 = \frac{5}{16N \sqrt{T} }(\sigma^2+6NG^2)+ \frac{15B^2}{8 \sqrt{T}}$.
\par

Assume $\hat{p} = \max_{t \in [T]} \sqrt{p_t}$, $\tilde{p} = \frac{1}{T}\sum_{t=0}^{T-1} \frac{1} {\sqrt{p_t}}$, and $V=8L (f(\bm w_0) - f(\bm w^*))$ . Summing Eq. (\ref{near final}) from $t=0,\dots,T-1$, the following result  reaches our final conclusion: 
\begin{equation}
	\begin{split}
		\frac{1}{T}\sum_{t=0}^{T-1}	\mathbb{E} ||    \nabla f \left(   {\bm {w}}_{t} \right) ||^2 & \leq \frac{\hat{p} V +\tilde{p} C_1+C_2 }{ \sqrt{T} } \\
	\end{split}
\end{equation}
\par
This shows the claim. 
{\color{blue}
}

\end{document}